%% file: main.tex
\documentclass{article}
\usepackage{iclr2022_conference,times}
\PassOptionsToPackage{dvipsnames}{xcolor}
\PassOptionsToPackage{round}{natbib}

\usepackage[utf8]{inputenc} %
\usepackage[T1]{fontenc}    %
\usepackage{hyperref}       %
\usepackage{url}            %
\usepackage{booktabs}       %
\usepackage{amsfonts}       %
\usepackage{nicefrac}       %
\usepackage{microtype}      %
\usepackage{xcolor}         %
\input{math_commands.tex}

\usepackage[dvipsnames]{xcolor}
\usepackage{mathtools}
\usepackage{amsmath}
\usepackage{amsthm}
\usepackage{mathtools}
\usepackage{enumerate}
\usepackage{graphicx}
\usepackage[export]{adjustbox}
\usepackage{booktabs}
\usepackage{tabu}
\usepackage{tikz}
\usepackage{tikzsymbols}
\usepackage{pgfplots}
\usepgfplotslibrary{groupplots}
\usepackage{listings}
\usepackage{multirow}
\usepackage{wrapfig}
\usepackage{subcaption}
\usepackage[toc,page]{appendix}
\usepackage[normalem]{ulem}
\usepackage{paralist}
\usepackage{cleveref}
\usepackage{makecell}
\usepackage{enumitem}
\usepackage{xspace}
\usepackage{array}
\usepackage{hyperref}

\newcommand{\kchoose}{{\scshape DictionaryLookup}}
\newcommand{\newgat}{GATv2}
\newcommand{\dpgat}{DPGAT}
\newcommand{\intn}{\left[n\right]}
\newcommand{\intm}{\left[m\right]}

\renewcommand\cite{\citep}
\newcommand\para[1]{\vspace{3pt}\noindent \textbf{#1}}

\newcommand{\err}[1]{\scriptsize{$\scriptscriptstyle{\pm}$#1}}

\newcommand{\PreserveBackslash}[1]{\let\temp=\\#1\let\\=\temp}
\newcolumntype{C}[1]{>{\PreserveBackslash\centering}p{#1}}
\newcolumntype{R}[1]{>{\PreserveBackslash\raggedleft}p{#1}}
\newcolumntype{L}[1]{>{\PreserveBackslash\raggedright}p{#1}}

\newtheorem{theorem}{Theorem}
\theoremstyle{definition}
\newtheorem{definition}{Definition}[section]

\newtheoremstyle{case}{}{}{}{}{}{:}{ }{}
\theoremstyle{case}

\DeclarePairedDelimiter\set\{\}

\DeclarePairedDelimiter\abs{\lvert}{\rvert}

\title{How Attentive are Graph Attention \\ Networks? 
}

\author{
\!Shaked Brody \\
	Technion \\
	\texttt{shakedbr@cs.technion.ac.il}
	\And
	Uri Alon \\
	Language Technologies Institute \\ 
	Carnegie Mellon University \\
	\texttt{ualon@cs.cmu.edu}
	\And
	Eran Yahav \\
	Technion \\
	\texttt{yahave@cs.technion.ac.il}
}
\iclrfinalcopy 
\begin{document}
\maketitle

\input{abstract.tex}
\input{01-intro.tex}

\input{background.tex}
\input{attention.tex}

\input{evaluation.tex}

\input{related.tex}
\input{conclusion.tex}
\input{acknowledgements.tex}

\begin{small}
	\bibliography{bib}
\end{small}

\newpage
\appendix
\input{appendix.tex}

\end{document}

%% file: math_commands.tex
\usepackage{amsmath,amsfonts,bm}

\def\eqref#1{equation~\ref{#1}}

\def\1{\bm{1}}

\def\va{{\bm{a}}}
\def\vb{{\bm{b}}}

\def\vh{{\bm{h}}}

\def\vk{{\bm{k}}}

\def\vq{{\bm{q}}}

\def\vx{{\bm{x}}}
\def\vy{{\bm{y}}}

\def\mK{{\bm{K}}}

\def\mM{{\bm{M}}}

\def\mP{{\bm{P}}}
\def\mQ{{\bm{Q}}}

\def\mU{{\bm{U}}}
\def\mV{{\bm{V}}}
\def\mW{{\bm{W}}}
\def\mX{{\bm{X}}}

\DeclareMathAlphabet{\mathsfit}{\encodingdefault}{\sfdefault}{m}{sl}
\SetMathAlphabet{\mathsfit}{bold}{\encodingdefault}{\sfdefault}{bx}{n}

\def\sK{{\mathbb{K}}}

\def\sQ{{\mathbb{Q}}}

\newcommand{\softmax}{\mathrm{softmax}}

%% file: abstract.tex
\begin{abstract}
Graph Attention Networks (GATs) are one of the most popular GNN architectures and are considered as the state-of-the-art architecture for representation learning with graphs. In GAT, every node attends to its neighbors given its own representation as the query.
However, in this paper we show that GAT computes a very limited kind of attention: the ranking of the attention scores is \emph{unconditioned on the query node}. We formally define this restricted kind of attention as \emph{static} attention and distinguish it from a strictly more expressive \emph{dynamic} attention.
Because GATs use a \emph{static} attention mechanism, there are simple graph problems that GAT cannot express: in a controlled problem, we show that static attention hinders GAT from even fitting the training data. 
To remove this limitation, we introduce a simple fix by modifying the order of operations
and propose \newgat{}: a \emph{dynamic} graph attention variant that is strictly more expressive than GAT. We perform an extensive evaluation and show that \newgat{} 
outperforms GAT across 12 OGB and other benchmarks while we match their parametric costs. 
Our code is available at \url{https://github.com/tech-srl/how_attentive_are_gats}.\footnote{An annotated implementation of \newgat{} is available at \url{https://nn.labml.ai/graphs/gatv2/}}
\newgat{} is available as part of the PyTorch Geometric library,\footnote{\texttt{from torch\_geometric.nn.conv.gatv2\_conv import GATv2Conv}} 
the Deep Graph Library,\footnote{\texttt{from dgl.nn.pytorch import GATv2Conv}}
and the TensorFlow GNN library.\footnote{\texttt{from tensorflow\_gnn.graph.keras.layers.gat\_v2 import GATv2Convolution}}
\end{abstract}

%% file: 01-intro.tex
\section{Introduction}\label{sec:intro}

Graph neural networks \cite[GNNs;][]{gori2005new,scarselli2008graph} have seen increasing popularity over the past few years \cite{duvenaud2015convolutional,atwood2016diffusion,bronstein2017geometric,monti2017geometric}. 
GNNs provide a general and efficient framework to learn from graph-structured data. 
Thus, GNNs are easily applicable in domains where the data can be represented as a set of nodes and the prediction depends on the relationships (edges) between the nodes. Such domains include molecules, social networks, product recommendation, computer programs and more.

In a GNN,
each node iteratively updates its state by interacting with its neighbors. GNN variants \cite{wu2019simplifying,xu2018powerful,li2015gated} mostly differ in how each node aggregates and combines the representations of its neighbors with its own. 
\citet{velic2018graph} pioneered the use of attention-based neighborhood aggregation, in one of the most common GNN variants --  Graph Attention Network (GAT).
In GAT, every node updates its representation by attending to its neighbors using its own representation as the query.
This generalizes the standard averaging or max-pooling of neighbors \cite{kipf2016semi,hamilton2017inductive}, by allowing every node 
to compute a \emph{weighted} average of its neighbors, 
and (softly) select its most relevant neighbors. 
The work of \citeauthor{velic2018graph} also generalizes the Transformer's \cite{vaswani2017attention} self-attention mechanism, from sequences to graphs \cite{joshi2020transformers}. 

Nowadays, GAT is one of the most popular GNN architectures~\cite{bronstein2021geometric}
and is considered as the state-of-the-art neural architecture for learning with graphs \cite{wang2019improving}.
Nevertheless, in this paper we show that \emph{GAT does not actually compute the expressive, well known, type of attention} \cite{bahdanau14}, which we call \emph{dynamic} attention. 
Instead, we show that GAT computes only a restricted ``static'' form of attention: for any query node, the attention function is \emph{monotonic} with respect to the neighbor (key) scores. That is,
the ranking (the $argsort$) of attention coefficients is shared across all nodes in the graph, and is \emph{unconditioned} on the query node. 
This fact severely hurts the expressiveness of GAT, and 
is demonstrated in \Cref{fig:gat-heatmap}.

\input{intro_fig.tex}

Supposedly, the conceptual idea of attention as the form of interaction 
between GNN nodes is orthogonal to the specific choice of attention function. 
However, \citeauthor{velic2018graph}'s original design of GAT has spread to a variety of domains 
\cite{wang2019improving,Yang_2020_CVPR,wang2019heterogeneous,huang2019syntax,ma2020entity,kosaraju2019social,nathani2019learning,wu2020comprehensive,Zhang2020Adaptive} and has become the default implementation of ``graph attention network'' in all popular GNN libraries such as PyTorch Geometric \cite{fey2019pytorchgeometric}, DGL \cite{wang2019dgl}, %
and others \cite{dwivedi2020benchmarkgnns,pytorchgat,brockschmidt2019graph}.

To overcome the limitation we identified in GAT, we introduce a simple fix to its attention function by only modifying the order of internal operations. The result is \newgat{} -- a graph attention variant that has a universal approximator attention function, %
and is thus \emph{strictly more expressive than GAT}. The effect of fixing the attention function in \newgat{} is demonstrated in \Cref{fig:newgat-heatmap}.

In summary, our main contribution is identifying that one of the most popular GNN types, the graph attention network, does not compute dynamic attention, the kind of attention that it seems to compute.
We introduce formal definitions for analyzing the expressive power of graph attention mechanisms (\Cref{def:static,def:dynamic}),
and derive our claims theoretically (\Cref{theorem:monotonic}) from the equations of \citet{velic2018graph}. 
Empirically, we use a synthetic problem to show that standard GAT \emph{cannot express} problems that require \emph{dynamic} attention (\Cref{subsec:kchoose}).
We introduce a simple fix 
by switching the order of internal operations in GAT, and propose \newgat{}, which \emph{does} compute dynamic attention (\Cref{theorem:dynamic}).
We further conduct a thorough empirical comparison of GAT and \newgat{} 
and find that \newgat{} outperforms GAT across 12 benchmarks of node-, link-, and graph-prediction.
For example, \newgat{} outperforms extensively tuned GNNs by over 1.4\% in the difficult ``UnseenProj Test'' set of the VarMisuse task \cite{allamanis2018learning}, without any hyperparameter tuning; 
and \newgat{} improves over an extensively-tuned GAT by 11.5\% in 13 prediction objectives in QM9.
In node-prediction benchmarks from OGB \cite{hu2020open}, not only that \newgat{} outperforms GAT with respect to accuracy -- we find that dynamic attention provided a much better robustness to noise. %

%% file: intro_fig.tex
\newcommand{\introfigheight}{11cm} 

\begin{figure*}[t]
        \centering
        \begin{subfigure}{.50\linewidth}
            \centering
            \includegraphics[height=\introfigheight]{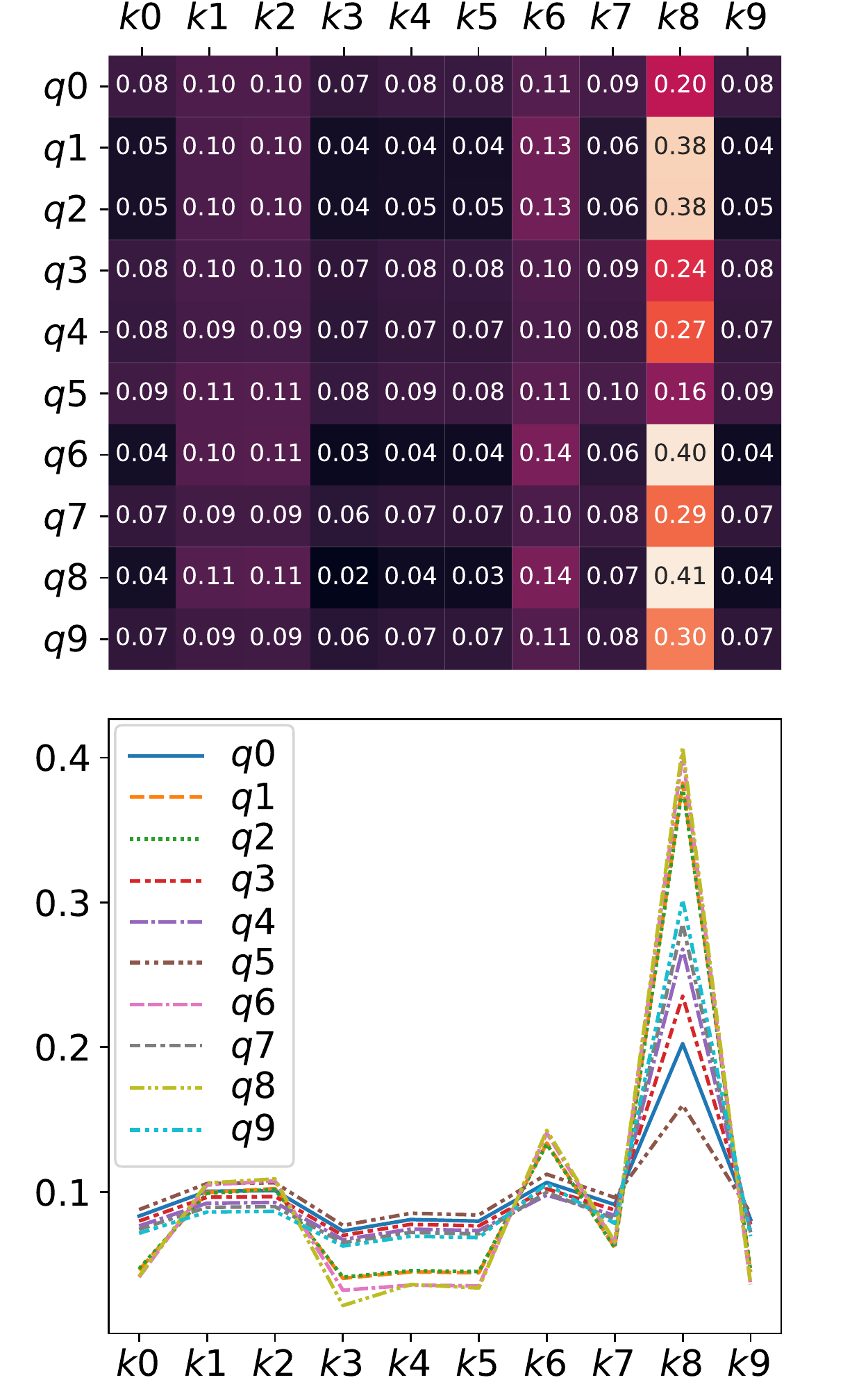}
			\caption{Attention in standard GAT (\citet{velic2018graph})}
            \label{fig:gat-heatmap}
        \end{subfigure}
        \hfill
        \begin{subfigure}{.48\linewidth}
            \centering
            \includegraphics[height=\introfigheight]{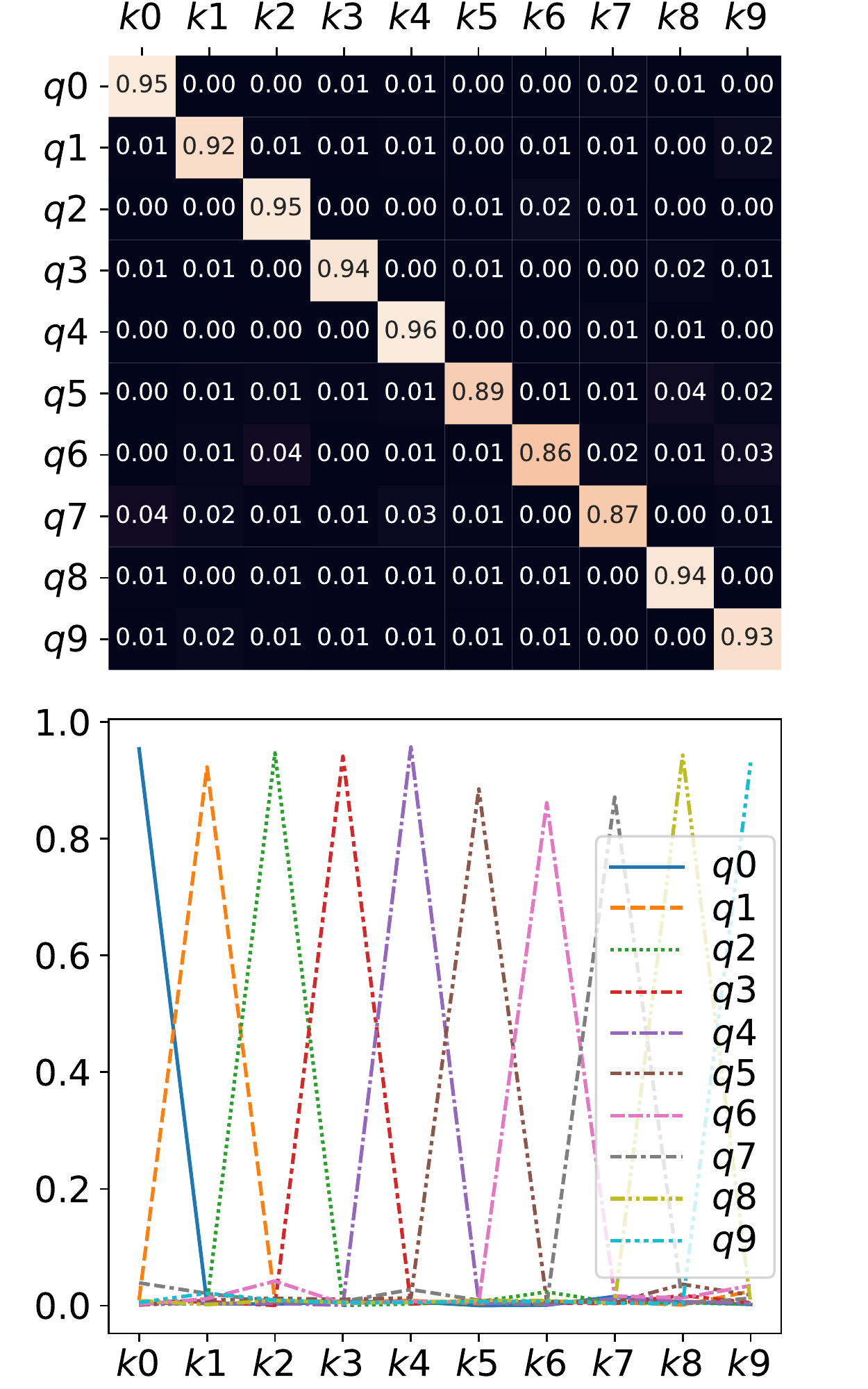}
			\caption{Attention in \newgat{}, our fixed version of GAT}
            \label{fig:newgat-heatmap}
        \end{subfigure}
        \caption{In a complete bipartite graph 
        of ``query nodes'' $\set{q0,...,q9}$ and ``key nodes'' $\set{k0,...,k9}$: standard GAT (\Cref{fig:gat-heatmap}) computes \emph{static} attention -- the ranking of attention coefficients is global for all nodes in the graph, and is unconditioned on the query node. For example, all queries ($q0$ to $q9$) attend mostly to the 8th key ($k8$). In contrast, \newgat{} (\Cref{fig:newgat-heatmap}) can actually compute \emph{dynamic} attention, where every query has a different ranking of attention coefficients of the keys.}
        \label{fig:heatmaps-dpgat-vs-gat}
\end{figure*}

%% file: background.tex
\section{Preliminaries}
\label{sec:background}
A directed graph $\mathcal{G}=\left(\mathcal{V},\mathcal{E}\right)$ contains nodes $\mathcal{V}=\set{1,...,n}$ and edges $\mathcal{E}\subseteq \mathcal{V}\times \mathcal{V}$, where $\left(j,i\right)\in\mathcal{E}$ denotes an edge from a node $j$ to a node $i$. We assume that every node $i\in\mathcal{V}$ has an initial representation $\vh_i^{\left(0\right)} \in \mathbb{R}^{d_0}$.
An undirected graph can be represented  with bidirectional edges.

\subsection{Graph Neural Networks}
A graph neural network (GNN) layer updates every node representation by aggregating its neighbors' representations. 
A layer's input is a set of node representations $\{\vh_i\in \mathbb{R}^{d} \mid i\in \mathcal{V}\}$ and the set of edges $\mathcal{E}$. A layer outputs a new set of node representations $\{\vh'_i\in \mathbb{R}^{d'} \mid i\in \mathcal{V}\}$, where the same parametric function is applied to every node given its neighbors $\mathcal{N}_i=\{j \in \mathcal{V} \mid \left( j,i \right) \in \mathcal{E}\}$:
\begin{equation}
	\vh'_i=f_{\theta}\left(
	\vh_i, 
	\mathrm{\scriptstyle{AGGREGATE}}\left(\set{\vh_j \mid j\in\mathcal{N}_i}\right)
	 \right)
	\label{eq:layer}
\end{equation}
The design of $f$ and $\mathrm{\scriptstyle{AGGREGATE}}$ is what mostly distinguishes one type of GNN from the other. 
For example, %
a common variant of GraphSAGE \cite{hamilton2017inductive} performs an element-wise mean as $\mathrm{\scriptstyle{AGGREGATE}}$, followed by concatenation with $\vh_i$,   a linear layer and a ReLU as $f$.

\subsection{Graph Attention Networks}
\label{subsec:gats}
GraphSAGE and many other popular GNN architectures \cite{xu2018powerful,duvenaud2015convolutional} weigh all neighbors $j\in\mathcal{N}_i$ with \emph{equal importance} (e.g., mean or max-pooling as $\mathrm{\scriptstyle{AGGREGATE}}$).
To address this limitation, GAT \cite{velic2018graph} instantiates \Cref{eq:layer} by computing a learned weighted average of the representations of $\mathcal{N}_i$. 
A scoring function $e: \mathbb{R}^{d}$$\,\times\,$$\mathbb{R}^{d}$$\,\rightarrow\,$$ \mathbb{R}$ computes a score %
for every edge $\left(j,i\right)$,
which indicates the importance of the features of the neighbor $j$ to the node $i$:
\begin{equation}
	e\left(\vh_i, \vh_j\right)=
	\mathrm{LeakyReLU}
	\left(
		\va^{\top}\cdot\left[\mW\vh_i \| \mW\vh_j\right]
	\right)
	\label{eq:gat}
\end{equation}
where $\va\in \mathbb{R}^{2d'}$, $\mW\in \mathbb{R}^{d'\times d}$ are learned, %
and $\|$ denotes vector concatenation.
These attention scores are normalized across all neighbors $j\in \mathcal{N}_i$ using softmax, and the attention function is defined as:
\begin{equation}
	\alpha_{ij} =
	\softmax_j\left(
	e\left(\vh_i, \vh_{j}\right)
	\right) =
	\frac{\mathrm{exp}\left(e\left(\vh_i, \vh_j\right)\right)}{\sum\nolimits_{j'\in\mathcal{N}_i} \mathrm{exp}\left(e\left(\vh_i, \vh_{j'}\right)\right)}
	\label{eq:softmax}
\end{equation}
Then, GAT computes a weighted average of the transformed features of the neighbor nodes (followed by a nonlinearity $\sigma$) as the new representation of $i$, using the normalized attention coefficients:
\begin{equation}
	\vh'_i=\sigma
	\left(
		\sum\nolimits_{j\in\mathcal{N}_{i}}
		\alpha_{ij}
		\cdot\mW\vh_j
	\right)
	\label{eq:weighted_avg}
\end{equation}
 From now on, we will refer to \Cref{eq:softmax,eq:weighted_avg,eq:gat} as the definition of GAT. 

%% file: attention.tex
\section{The Expressive Power of Graph Attention Mechanisms}
\label{sec:expressive}
In this section, we explain why attention is limited when it is not \emph{dynamic} (\Cref{subsec:dynamic}). We then show that GAT is severely constrained, 
because it can only compute \emph{static} attention (\Cref{subsec:limitation}). Next, we show how GAT can be fixed (\Cref{subsec:building}), by simply modifying the order of operations. 

We refer to a neural architecture (e.g., the scoring or the attention function of GAT) as a \emph{family of functions}, parameterized by the learned parameters. An element in the family is a concrete function with specific trained weights. %
In the following, we use $\intn$ to denote the set $\intn = \set{1, 2, ..., n} \subset \mathbb{N}$.

\subsection{The Importance of Dynamic Weighting}
\label{subsec:dynamic}
Attention is a mechanism for computing a distribution over a set of input \emph{key} vectors, given an additional \emph{query} vector. 
If the attention function always weighs one key at least as much as any other key, 
\emph{unconditioned on the query}, we say that this attention function is \emph{static}:
\begin{definition}[Static attention]
A (possibly infinite) family of scoring functions $\mathcal{F}$$\,\subseteq\,$$\left(\mathbb{R}^{d}\times \mathbb{R}^{d}\rightarrow \mathbb{R}\right)$ computes \emph{static scoring} for a given set of key vectors $\sK$$\,=\,$$\set{\vk_1, ..., \vk_{n}}$$\,\subset\,$$ \mathbb{R}^{d}$ and query vectors $\sQ$$\,=\,$$\set{\vq_1, ..., \vq_{m}}$$\,\subset\,$$ \mathbb{R}^{d}$,
	if for every $f \in \mathcal{F}$
	there exists a ``highest scoring'' key $j_{f} \in \intn$ %
	such that for every query $i \in \intm$ and key $j \in \intn$ it holds that
	$f\left(\vq_i, \vk_{j_{f}}\right) \geq f\left(\vq_i, \vk_{j}\right)$. 
	We say that a family of attention functions computes \emph{static attention}  given $\mathbb{K}$ and $\mathbb{Q}$, 
if its scoring function computes static scoring, possibly followed by monotonic normalization such as softmax.
	\label{def:static}
\end{definition}
Static attention is very limited because every function $f\in\mathcal{F}$ has 
a key that is \emph{always selected}, regardless of the query.
Such functions  cannot model situations where different keys have different relevance to different queries.  Static attention is demonstrated in \Cref{fig:gat-heatmap}.

The general and powerful form of attention is \emph{dynamic attention}:
\begin{definition}[Dynamic attention]
	\input{def-dynamic.tex}

	\label{def:dynamic}
\end{definition}
That is, dynamic attention %
can \emph{select} every key $\varphi\left(i\right)$ using the  query $i$, by making $f\left(\vq_{i}, \vk_{\varphi\left(i\right)}\right)$ the maximal in $\set{f\left(\vq_{i}, \vk_{j}\right) \mid j\in \intn}$. 
Note that \emph{dynamic} and \emph{static} attention are exclusive properties, but they are not complementary. Further, every \emph{dynamic} attention family has strict subsets of \emph{static} attention families with respect to the same $\mathbb{K}$ and $\mathbb{Q}$.
Dynamic attention is demonstrated in \Cref{fig:newgat-heatmap}.

\para{Attending by decaying}
Another way to think about attention is the ability to ``focus'' on the most relevant inputs, given a query. Focusing is only possible by \emph{decaying} other inputs, i.e., giving these decayed inputs lower scores than others. %
If one key is always given an equal or greater attention score than other keys (as in static attention), no query can  ignore this key or decay this key's score.

\subsection{The Limited Expressivity of GAT}
\label{subsec:limitation}
Although the scoring function $e$  can be defined in various ways,  
the original definition of
\citet{velic2018graph} (\Cref{eq:gat}) has become the \emph{de facto} practice: it has spread to a variety of domains 
and is now the standard implementation of ``graph attention network'' in all popular GNN libraries \cite{fey2019pytorchgeometric, wang2019dgl, dwivedi2020benchmarkgnns, pytorchgat, brockschmidt2019graph}.

The motivation of GAT is to compute a representation for every node as a weighted average of its neighbors. 
Statedly, GAT is inspired by the attention mechanism of \citet{bahdanau14} and the self-attention mechanism of the Transformer \cite{vaswani2017attention}. 
Nonetheless:
\begin{theorem}
A GAT layer computes only static attention, for any set of node representations $\mathbb{K}=\mathbb{Q}=\set{\vh_1,...,\vh_n}$. 
	In particular, for $n$$\,>\,$$1$,
a GAT layer does not compute dynamic attention.%
\label{theorem:monotonic}	
\end{theorem}
\begin{proof} 
Let $\mathcal{G}=\left(\mathcal{V},\mathcal{E}\right)$ be a graph modeled by 
a GAT layer with 
some $\va$ and $\mW$ values (\Cref{eq:softmax,eq:gat}), and having node representations $\set{\vh_1,..., \vh_n}$.
The learned parameter $\va$ can be written as a concatenation $\va=\left[\va_1 \| \va_2\right]\in\mathbb{R}^{2d'}$ such that $ \va_1,\va_2\in\mathbb{R}^{d'}$, and \Cref{eq:gat} can be re-written as:
\begin{equation}
	e\left(\vh_{i}, \vh_{j}\right)=
	\mathrm{LeakyReLU}
	\left(
		\va_1^{\top}\mW\vh_{i} + \va_2^{\top}\mW\vh_{j}
	\right)
	\label{eq:gat-re}
\end{equation}
Since $\mathcal{V}$ is finite, 
there exists a node $j_{max}\in\mathcal{V}$ such that $\va_2^{\top}\mW\vh_{j_{max}}$ is maximal among all nodes $j\in\mathcal{V}$ ($j_{max}$ is the $j_f$ required by \Cref{def:static}).
Due to the monotonicity of $\mathrm{LeakyReLU}$ and $\mathrm{softmax}$, for every query node $i \in\mathcal{V}$, the node $j_{max}$ also leads to the maximal value of its attention distribution
$\set{\alpha_{ij} \mid j \in \mathcal{V}}$. 
Thus, from \Cref{def:static} directly, $\alpha$ computes
 only \emph{static attention}. 
 This also implies that $\alpha$ does not compute dynamic attention, because in GAT, \Cref{def:dynamic} holds only for \emph{constant} mappings $\varphi$ that map all inputs to the same output.
\end{proof}
The consequence of \Cref{theorem:monotonic} is that for any set of nodes $\mathcal{V}$ and a trained GAT layer, 
the attention function $\alpha$
defines a constant ranking ($argsort$) of the nodes, unconditioned on the query nodes $i$. That is, we can denote $s_j=\va_2^{\top}\mW\vh_{j}$ and get that for any choice of $\vh_i$, $\alpha$ is monotonic with respect to the per-node scores $\set{s_j\mid j \in \mathcal{V} }$. 
This global ranking induces the local ranking of every neighborhood $\mathcal{N}_i$.
The only effect of $\vh_i$ is in the ``sharpness'' of the produced attention distribution. This is demonstrated in \Cref{fig:gat-heatmap} (bottom), where different curves denote different queries ($\vh_i$). %

\para{Generalization to multi-head attention} 
\citet{velic2018graph} found it beneficial to employ $H$ separate attention heads and concatenate their outputs, 
similarly to Transformers.
In this case, \Cref{theorem:monotonic} holds for each head separately:  every head $h\in \left[H\right]$ has a (possibly different) node that maximizes 
$\set{s_j^{\left(h\right)}\mid j \in \mathcal{V} }$
, and the output is the concatenation of $H$ static attention heads.

\subsection{Building Dynamic Graph Attention Networks}
\label{subsec:building}
To create a \emph{dynamic} graph attention network, 
we modify the order of internal operations in GAT and introduce \newgat{} -- a simple fix of GAT that 
has a strictly more expressive attention mechanism. %

\para{\newgat} The main problem in the standard GAT scoring function (\Cref{eq:gat}) is that the learned layers $\mW$ and $\va$ are applied consecutively, and thus can be collapsed into a \emph{single} linear layer. %
To fix this limitation, we simply apply the $\va$ layer \emph{after} the nonlinearity ($\mathrm{LeakyReLU}$), and  the $\mW$ layer after the concatenation,
\footnote{We also add a bias vector $\vb$ before applying the nonlinearity, we omit this in \Cref{eq:gat2-vs-gat} for brevity. } 
effectively applying an MLP to compute the score for each query-key pair:
\begin{align}
&			 \mathrm{GAT} \text{ \cite{velic2018graph}:} &
e\left(\vh_i, \vh_j\right)= &
	\mathrm{LeakyReLU}
	\left(
		\va^{\top}\cdot\left[\mW\vh_i \| \mW\vh_j\right]
	\right)
	\label{eq:gat-vs-gat2}  
\\
	&			 \mathrm{\newgat{}} \text{ (our fixed version):} &
 e
 \left(\vh_{i}, \vh_{j}\right)  = &
\va
^{\top}
	\mathrm{LeakyReLU}
	\left(
		\mW \cdot \left[\vh_{i} \| \vh_{j}\right] 
	\right) \label{eq:gat2-vs-gat}
\end{align}
The simple modification %
makes a significant difference in the expressiveness of the attention function:
\begin{theorem}
	\input{theorem-gat2.tex}
	\label{theorem:dynamic}
\end{theorem}

We prove \Cref{theorem:dynamic} in \Cref{sec:proof-dynamic}.
The main idea is that we can define an appropriate function that  \newgat{} will be a universal approximator
\cite{cybenko1989approximation, hornik1991approximation} of. In contrast, GAT (\Cref{eq:gat-vs-gat2}) cannot approximate any such desired function (\Cref{theorem:monotonic}).

\para{Complexity}
\newgat{} has the same time-complexity as GAT's declared complexity:~$\mathcal{O}\left(\abs{\mathcal{V}}dd' + \abs{\mathcal{E}}d'\right)$. However, by merging its linear layers, GAT can %
be computed faster
than stated by \citet{velic2018graph}.
For a detailed time- and parametric-complexity analysis, %
see \Cref{sec:complexity}.

%% file: def-dynamic.tex
	A (possibly infinite) family of scoring functions $\mathcal{F}$$\,\subseteq\,$$\left(\mathbb{R}^{d}\times \mathbb{R}^{d}\rightarrow\mathbb{R}\right)$ computes \emph{dynamic scoring} for a given set of key vectors $\sK$$\,=\,$$\set{\vk_1, ..., \vk_{n}}$$\,\subset\,$$ \mathbb{R}^{d}$ and query vectors $\sQ$$\,=\,$$\set{\vq_1, ..., \vq_{m}}$$\,\subset\,$$ \mathbb{R}^{d}$,
	if for \emph{any} mapping $\varphi$$: \intm \rightarrow \intn$ there exists $f \in \mathcal{F}$ such that 
	for any query  $i \in \intm$ 
	and any key $j_{\neq \varphi\left(i\right)}  \in \intn$:
	$f\left(\vq_{i}, \vk_{\varphi\left(i\right)}\right) > f\left(\vq_{i}, \vk_{j}\right)$. 
We say that a family of attention functions computes \emph{dynamic attention} for $\mathbb{K}$ and $\mathbb{Q}$, 
if its scoring function computes dynamic scoring, possibly followed by monotonic normalization such as softmax.

%% file: theorem-gat2.tex
A \newgat{} layer computes dynamic attention
for any set of node representations $\mathbb{K}=\mathbb{Q}=\set{\vh_1,...,\vh_n}$
 .

%% file: evaluation.tex
\section{Evaluation}
\label{sec:eval}
First, we demonstrate the weakness of GAT using a simple synthetic problem that GAT cannot even fit (cannot even achieve high \emph{training} accuracy), but is easily solvable by \newgat{} (\Cref{subsec:kchoose}).
Second, we show that \newgat{} is much more \emph{robust to edge noise}, because its dynamic attention mechanisms allow it to decay noisy (false) edges, while GAT's performance severely decreases as noise increases (\Cref{subsec:robustness}).
Finally, we compare GAT and \newgat{}
across 12 benchmarks overall. %
(\Cref{subsec:node,subsec:varmisuse,subsec:link,subsec:graph,subsec:pubmed}). 
We find that GAT is inferior to \newgat{} across all examined benchmarks.

\para{Setup}
When previous results exist, we take hyperparameters that were tuned for GAT and use them in  \newgat{}, without any additional tuning.
Self-supervision \cite{kim2021how,rong2020self}, graph regularization 
\cite{Zhao2020PairNorm,rong2020dropedge}, and other tricks 
\cite{wang2021bag,huang2021combining} are orthogonal to the contribution of the GNN layer itself, and may further improve all GNNs.
In all experiments of \newgat{}, we constrain the learned matrix by setting $\mW=\left[\mW' \| \mW'\right]$, to rule out the increased number of parameters over GAT as the source of empirical difference
 (see \Cref{subsec:parametric-cost}).
Training details, statistics, and code are provided in \Cref{sec:training}. %

Our main goal is to compare dynamic and static graph attention mechanisms.
However, for reference, we also include non-attentive baselines such as GCN \cite{kipf2016semi}, GIN \cite{xu2018powerful} and GraphSAGE \cite{hamilton2017inductive}. These non-attentive GNNs can be thought of as a special case of attention, where every node gives all its neighbors the same attention score.
Additional comparison to a Transformer-style scaled dot-product attention (``\dpgat{}''), which is \emph{strictly weaker} than our proposed \newgat{} (see a proof in \Cref{sec:dpgat-proof}),  is shown in \Cref{sec:dpgat}.

\input{kchoose.tex}

\input{noise.tex}
\input{varmisuse.tex}

\input{node-prediction.tex}

\input{graph-prediction.tex}

\input{link-prediction.tex}

\subsection{Discussion}
In \emph{all} examined benchmarks, we found that \emph{\newgat{} is more accurate than GAT}. 
Further, we found that \newgat{} is significantly more robust to noise than GAT.
In the synthetic \kchoose{} benchmark (\Cref{subsec:kchoose}), GAT fails to express the data, and thus achieves even poor \emph{training} accuracy.

In few of the benchmarks (\Cref{tab:link-results} and some of the properties in \Cref{tab:qm-results}) -- a non-attentive model such as GCN or GIN  achieved a higher accuracy than all GNNs that do use attention. %

\para{Which graph attention mechanism should I use?} 
It is usually impossible to determine in advance which architecture would perform best. A theoretically weaker model may perform better in practice, because a stronger model might overfit the training data if the task is ``too simple'' and does not require such expressiveness.
Intuitively, we believe that the 
more complex the interactions between nodes are -- the more benefit a GNN can take from theoretically stronger graph attention mechanisms such as \newgat{}.
The main question is whether the problem has a \emph{global ranking} of ``influential'' nodes (GAT is sufficient), 
or do different nodes have \emph{different rankings} of  neighbors (use \newgat{}).

\citeauthor{veliconly}, the author of GAT, has confirmed on Twitter
\footnote{\url{https://twitter.com/PetarV_93/status/1399685979506675714}}   
 that GAT was designed to work in the ``easy-to-overfit'' datasets of the time (2017), such as Cora, Citeseer and Pubmed \cite{sen2008collective}, where the data might had an underlying static ranking of ``globally important'' nodes. 
\citeauthor{veliconly} agreed that newer and more challenging benchmarks may 
demand stronger attention mechanisms such as \newgat{}.
In this paper, we revisit the traditional assumptions and show that many modern graph benchmarks and datasets contain more complex interactions, and thus  \emph{require dynamic attention}. %

%% file: kchoose.tex
\input{kchoose_results.tex}
\subsection{Synthetic Benchmark: \kchoose{}}
\label{subsec:kchoose}
The \kchoose{} problem is a contrived problem that we designed to test the ability of a GNN architecture to perform dynamic attention. Here, we demonstrate that GAT cannot learn this simple problem.
\Cref{fig:kchoose} shows a complete bipartite graph of $2k$ nodes. 
Each ``key node'' in the bottom row has an \emph{attribute} ($\set{\mathrm{A,B,C,...}}$) and a \emph{value} ($\set{\mathrm{1,2,3,...}}$). 
Each ``query node'' in the upper row has \emph{only an attribute} ($\set{\mathrm{A,B,C,...}}$). %
The goal is to predict the value of every query node (upper row), according to its attribute. 
Each graph in the dataset has a different mapping from attributes to values.
We created a separate dataset for each $k=\set{1,2,3,...}$, for which we trained a different model, and measured per-node accuracy. 

Although this is a contrived problem, it is relevant to any subgraph with keys that share more than one query, and each query needs to attend to the keys differently. Such subgraphs are very common in a variety of real-world domains. 
This problem tests the layer itself because it can be solved using a \emph{single} GNN layer, without suffering from multi-layer side-effects such as over-smoothing \cite{li2018deeper}, over-squashing \cite{alon2021bottleneck}, or vanishing gradients \cite{li2019deepgcns}.
Our %
code will be made publicly available, to serve as a testbed for future graph attention mechanisms.

\para{Results} \Cref{fig:kchoose-results} shows the following surprising results: GAT with a single head (GAT$_{1h}$) failed to fit the \emph{training} set for any value of $k$, no matter for how many iterations it was trained, and after trying various training methods. Thus, it expectedly fails to generalize (resulting in low test accuracy).
Using 8 heads, GAT$_{8h}$ successfully fits the \emph{training} set, but generalizes \emph{poorly} to the \emph{test} set.
In contrast, \newgat{} easily achieves 100\% training and 100\% test accuracies for any value of $k$, and even for $k$$=$$100$ (not shown) and using a \emph{single head}, thanks to its ability to perform dynamic attention.
These results clearly show the limitations of GAT, which are easily solved by \newgat{}.
An additional comparison to GIN, which could \emph{not} fit this dataset, is provided in \Cref{fig:kchoose-results-additional} in \Cref{sec:additional-kchoose}.

\para{Visualization}
\Cref{fig:gat-heatmap} (top) shows a heatmap of GAT's attention scores in this \kchoose{} problem. As shown, all query nodes $q0$ to $q9$ 
attend mostly to the eighth key ($k8$), and have the same ranking of attention coefficients (\Cref{fig:gat-heatmap} (bottom)). In contrast, \Cref{fig:newgat-heatmap} shows how \newgat{} can \emph{select} a different key node for every query node, because it computes dynamic attention.

\para{The role of multi-head attention}
\citet{velic2018graph} found the role of multi-head attention to be stabilizing the learning process.
Nevertheless, \Cref{fig:kchoose-results} shows that increasing the number of heads 
strictly increases training accuracy, and thus, the expressivity. 
Thus, GAT \emph{depends} on having multiple attention heads. In contrast, even a \emph{single} \newgat{} head generalizes better than a multi-head GAT.

%% file: kchoose_results.tex
\begin{figure*}[t]
\begin{minipage}[t][][b]{0.35\textwidth}
\input{kchoose_figure.tex}	
\end{minipage}
\hfill
\begin{minipage}[t][][b]{0.62\textwidth}
\input{kchoose_plot.tex}
\end{minipage}
\vspace{-4mm}
\end{figure*}

%% file: kchoose_figure.tex
\centering
\includegraphics[height=25mm]{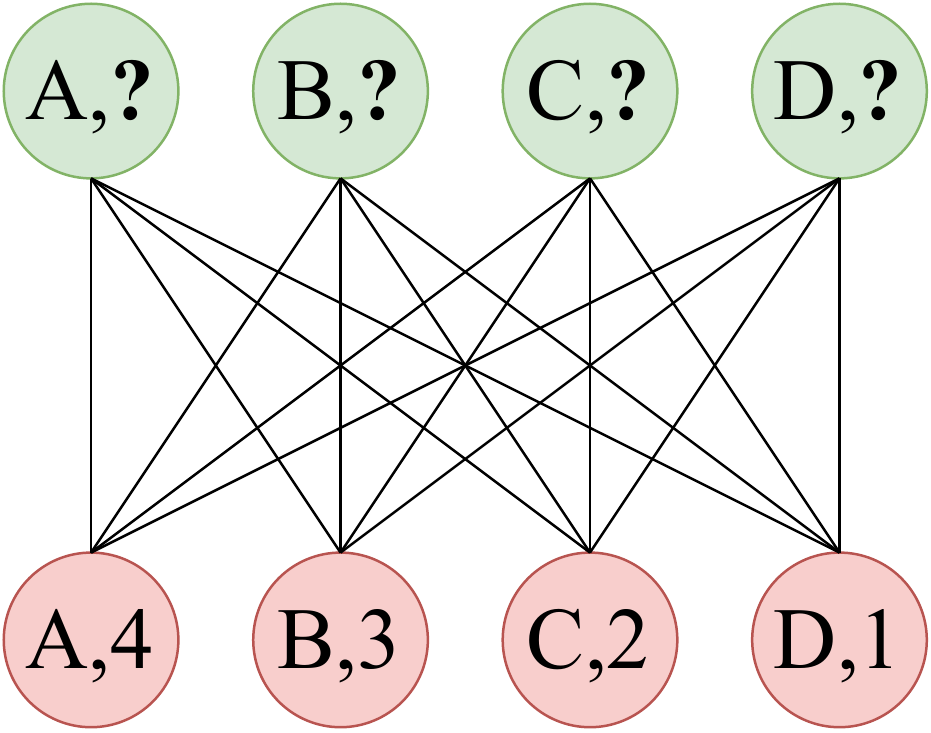}
\caption{The \kchoose{} problem of size $k$$=$$4$: every node in the bottom row has an alphabetic \emph{attribute} ($\set{\mathrm{A,B,C,...}}$) and a numeric \emph{value} ($\set{\mathrm{1,2,3,...}}$); every node in the upper row has only an attribute; the goal is to predict the value for each node in the upper row, using its attribute.
}
\label{fig:kchoose}

%% file: kchoose_plot.tex
\definecolor{ao}{rgb}{0.0, 0.5, 0.0}
\definecolor{mypurple}{HTML}{AB30C4}

	\begin{tikzpicture}[scale=1]
	    \definecolor{color0}{rgb}{0.917647058823529,0.917647058823529,0.949019607843137}
    \definecolor{color1}{rgb}{0.282352941176471,0.470588235294118,0.815686274509804}
    \definecolor{color3}{rgb}{0.933333333333333,0.52156862745098,0.290196078431373}
    \definecolor{color2}{rgb}{0.415686274509804,0.8,0.392156862745098}
	\begin{axis}[
		xlabel={$k$ (number of different keys in each graph)},
		ylabel={\footnotesize{Accuracy}},
		ylabel near ticks,
        legend style={at={(1,1)},anchor=north west,mark size=2pt,
        	font=\footnotesize, %
        	inner xsep=0pt, inner sep=1pt},
        legend cell align={left},
        xmin=3.5, xmax=20.2,
        ymin=-0.0, ymax=107,
        xtick={4,...,20},
        ytick={0,10,20,...,100},
        label style={font=\footnotesize},
        ylabel style={font=\footnotesize},
        ylabel shift={-5pt},        
        tick label style={font=\scriptsize} ,
        grid = major,
        major grid style={dotted,gray},
        width = 0.78\linewidth, height = 52.55mm
    ]

\addplot[color=color2, solid, mark options={solid, fill=color2, draw=black}, mark=triangle*, line width=0.5pt, mark size=3pt, visualization depends on=\thisrow{alignment} \as \alignment, nodes near coords, point meta=explicit symbolic,
    every node near coord/.style={anchor=\alignment, font=\scriptsize}] 
table [meta index=2]  {
x   y       label   alignment
4	100		{}		-129
5	100		{}		-90
6	100		{}		-90
7	100		{}		-60
8	100		{}		180
9	100		{}		180
10	100	{}	-90
11	100		{}		-30
12	100		{}		-30
13	100		{}		-30
14	100		{}		-30
15	100		{}		-30
16	100		{}		-30
17	100		{}		-30
18	100		{}		-30
19	100		{}		-30
20	100		{}		-30
	};
    \addlegendentry{\newgat{} test}

\addplot[color=color3, mark options={solid, fill=color3, draw=black}, mark=square*, line width=0.5pt, mark size=2pt, visualization depends on=\thisrow{alignment} \as \alignment, nodes near coords, point meta=explicit symbolic,
    every node near coord/.style={anchor=\alignment, font=\scriptsize}] 
table [meta index=2]  {
x   y       label   alignment
4	29		{}		-129
5	30		{}		-90
6	48		{}		-90
7	75		{}		-60
8	85		{}		180
9	67		{}		180
10	87	{}	-90
11	68		{}		-30
12	53		{}		-30
13	65		{}		-30
14	50		{}		-30
15	59		{}		-30
16	53		{}		-30
17	66		{}		-30
18	42		{}		-30
19	52		{}		-30
20	43		{}		-30
	};
    \addlegendentry{GAT$_{8h}$ test}      

\addplot[color=color1,densely dashed, mark options={solid, fill=color1, draw=black}, mark=*, line width=0.5pt, mark size=2pt, visualization depends on=\thisrow{alignment} \as \alignment, nodes near coords, point meta=explicit symbolic,
    every node near coord/.style={anchor=\alignment, font=\scriptsize}] 
table [meta index=2]  {
x   y       label   alignment
4	85		{}		-129
5	83		{}		-90
6	67		{}		-90
7	63		{}		-60
8	34		{}		180
9	29		{}		180
10	27		{}	-90
11	21		{}		-30
12	22		{}		-30
13	22		{}		-30
14	18		{}		-30
15	22		{}		-30
16	16		{}		-30
17	17		{}		-30
18	16		{}		-30
19	15		{}		-30
20	11		{}		-30
	};
    \addlegendentry{GAT$_{1h}$ train}
    
\addplot[color=color1, mark options={solid, fill=color1, draw=black}, mark=*, line width=0.5pt, mark size=2pt, visualization depends on=\thisrow{alignment} \as \alignment, nodes near coords, point meta=explicit symbolic,
    every node near coord/.style={anchor=\alignment, font=\scriptsize}] 
table [meta index=2]  {
x   y       label   alignment
4	45		{}		-129
5	49		{}		-90
6	43		{}		-90
7	41		{}		-60
8	32		{}		180
9	26		{}		180
10	20		{}	-90
11	18		{}		-30
12	17		{}		-30
13	21		{}		-30
14	17		{}		-30
15	16		{}		-30
16	13		{}		-30
17	14		{}		-30
18	15		{}		-30
19	13		{}		-30
20	11		{}		-30
	};
    \addlegendentry{GAT$_{1h}$ test}

	\end{axis}
\end{tikzpicture}
\caption{The \kchoose{} problem: \newgat{} easily achieves 100\% train and test accuracies even for $k$$=$$100$ and using only a single head.}
\label{fig:kchoose-results}

%% file: noise.tex
\subsection{Robustness to Noise}
\label{subsec:robustness}
We examine the robustness of \emph{dynamic} and \emph{static} attention to noise. In particular, we focus on structural noise:   %
given an input graph $\mathcal{G}$$\,=\,$$(\mathcal{V},\mathcal{E})$ and a noise ratio $0$$\,\le\,$$ p  $$\,\le\,$$1$, we randomly sample $\abs{\mathcal{E}}$$\times$$p$ non-existing edges $\mathcal{E}'$ from 
$\mathcal{V}$$\times$$\mathcal{V}$$\setminus$$\mathcal{E}$.
We then %
train the GNN on the noisy graph
$\mathcal{G}'$$=$$(\mathcal{V},\mathcal{E}\cup\mathcal{E}')$. %

\input{noise_fig.tex}

\para{Results}
\Cref{fig:noise} shows the accuracy on two node-prediction datasets from the Open Graph Benchmark \cite[OGB; ][]{hu2020open} %
as a function of the noise ratio $p$. 
As $p$ increases, all models show a natural decline in test accuracy in both datasets. 
Yet, thanks to their ability to compute \emph{dynamic} attention, \newgat{} shows a milder degradation in accuracy compared to GAT, which 
shows a steeper descent. %
We hypothesize that 
the ability to perform \emph{dynamic} attention helps the models 
distinguishing between given data edges ($\mathcal{E}$) and noise edges ($\mathcal{E'}$);
in contrast, GAT cannot distinguish between edges, because it scores the source and target nodes separately. 
These results clearly demonstrate the \emph{robustness} of \emph{dynamic} attention over \emph{static} attention in noisy settings, which are common in reality. %

%% file: noise_fig.tex
\newcommand{\noisefigsacle}{1} 
\newcommand{\noisefigheight}{60mm} 
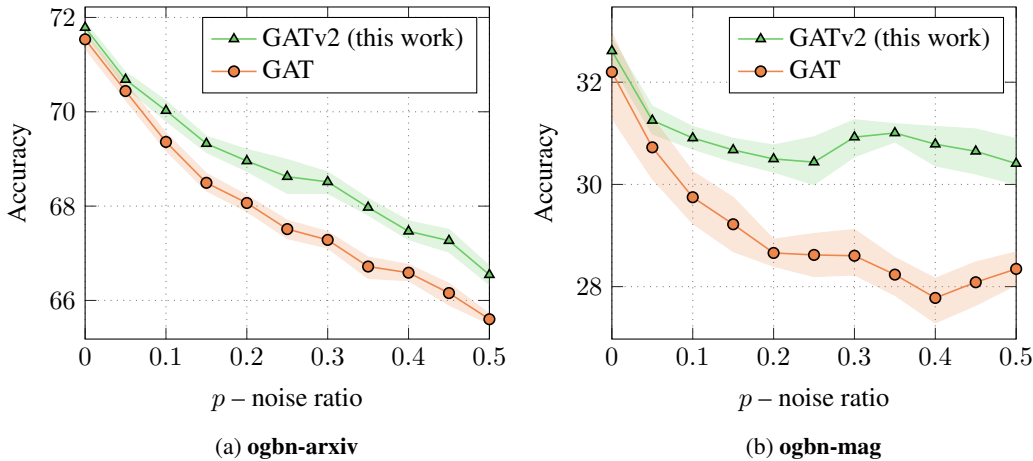
\begin{figure*}[h!]
    \centering
    \begin{subfigure}[b]{.54\textwidth}
		\centering
        \begin{tikzpicture}[trim axis left,trim axis right, scale=\noisefigsacle]

            \definecolor{color0}{rgb}{0.917647058823529,0.917647058823529,0.949019607843137}
            \definecolor{color1}{rgb}{0.282352941176471,0.470588235294118,0.815686274509804}
            \definecolor{color2}{rgb}{0.933333333333333,0.52156862745098,0.290196078431373}
            \definecolor{color3}{rgb}{0.415686274509804,0.8,0.392156862745098}
            
            \begin{axis}[
            axis line style={black},
            legend cell align={left},
            legend style={
			},
            legend pos=north east,
            legend entries={{\newgat} (this work),{GAT}},
            tick align=inside,
            tick pos=both,
            grid style={dotted, gray},
		    xlabel={$p$ -- noise ratio},
            xmajorgrids,
            xmin=-0.0, xmax=0.5,
            xtick style={color=white!15!black},
            ylabel={Accuracy},
            ymajorgrids,
            ymin=65.1810334331438, ymax=72.2220649577313,
            ytick style={color=white!15!black},
            ylabel near ticks,
			height=\noisefigheight,
            ]

            \path [draw=color3, fill=color3, opacity=0.2]
            (axis cs:0,71.90201807025)
            --(axis cs:0,71.655834636175)
            --(axis cs:0.05,70.514165497)
            --(axis cs:0.1,69.80446233675)
            --(axis cs:0.15,69.1386837789)
            --(axis cs:0.2,68.738708914725)
            --(axis cs:0.25,68.2689663314)
            --(axis cs:0.3,68.266126652025)
            --(axis cs:0.35,67.809358043175)
            --(axis cs:0.4,67.28881729125)
            --(axis cs:0.45,67.025620174775)
            --(axis cs:0.5,66.3337175726)
            --(axis cs:0.5,66.73437362935)
            --(axis cs:0.5,66.73437362935)
            --(axis cs:0.45,67.506019459275)
            --(axis cs:0.4,67.687276551675)
            --(axis cs:0.35,68.1653440689)
            --(axis cs:0.3,68.72463998525)
            --(axis cs:0.25,68.988544999675)
            --(axis cs:0.2,69.20296705065)
            --(axis cs:0.15,69.482360612)
            --(axis cs:0.1,70.2378646482)
            --(axis cs:0.05,70.8456386014)
            --(axis cs:0,71.90201807025)
            --cycle;
        
            \path [draw=color2, fill=color2, opacity=0.2]
            (axis cs:0,71.690460548925)
            --(axis cs:0,71.336918506325)
            --(axis cs:0.05,70.247515105)
            --(axis cs:0.1,69.177359618825)
            --(axis cs:0.15,68.323553693375)
            --(axis cs:0.2,67.8844814293)
            --(axis cs:0.25,67.31310886385)
            --(axis cs:0.3,67.110750140475)
            --(axis cs:0.35,66.46680670075)
            --(axis cs:0.4,66.4188961789)
            --(axis cs:0.45,65.9022476761)
            --(axis cs:0.5,65.501080320625)
            --(axis cs:0.5,65.72084407885)
            --(axis cs:0.5,65.72084407885)
            --(axis cs:0.45,66.362788716425)
            --(axis cs:0.4,66.762989407975)
            --(axis cs:0.35,66.918991967375)
            --(axis cs:0.3,67.459853879)
            --(axis cs:0.25,67.701022451975)
            --(axis cs:0.2,68.230992887375)
            --(axis cs:0.15,68.694935606)
            --(axis cs:0.1,69.55477395875)
            --(axis cs:0.05,70.615602721025)
            --(axis cs:0,71.690460548925)
            --cycle;
            
            \addplot [semithick, color3, mark=triangle*, mark options={draw=black}]
            table {%
            0 71.785281372
            0.05 70.68308487
            0.1 70.024689484
            0.15 69.324938203
            0.2 68.957263945
            0.25 68.624158478
            0.3 68.514289092
            0.35 67.972553253
            0.4 67.466616057
            0.45 67.266835785
            0.5 66.544245148
            };
            \addplot [semithick, color2, mark=*, mark options={draw=black}]
            table {%
            0 71.535913086
            0.05 70.439889525
            0.1 69.360738372
            0.15 68.494125365
            0.2 68.066168212
            0.25 67.512911225
            0.3 67.283707428
            0.35 66.718926241
            0.4 66.587866975
            0.45 66.156616211
            0.5 65.601711273
            };
            \end{axis}
        \end{tikzpicture}
        \caption{\textbf{ogbn-arxiv}}
        \label{fig:noise-arxiv}
    \end{subfigure}
	\begin{subfigure}[b]{.45\textwidth}
        \centering
\begin{tikzpicture}[trim axis left,trim axis right, scale=\noisefigsacle]

    \definecolor{color0}{rgb}{0.917647058823529,0.917647058823529,0.949019607843137}
    \definecolor{color1}{rgb}{0.282352941176471,0.470588235294118,0.815686274509804}
    \definecolor{color3}{rgb}{0.933333333333333,0.52156862745098,0.290196078431373}
    \definecolor{color2}{rgb}{0.415686274509804,0.8,0.392156862745098}
    
    \begin{axis}[
    axis line style={black},
    legend cell align={left},
    legend style={
      at={(0.03,0.03)},
      anchor=south west,
    },
    legend pos=north east,
    legend entries={{\newgat} (this work),{GAT}},
    tick align=inside,
    tick pos=both,
    grid style={dotted, gray},
    xlabel={$p$ -- noise ratio},
    xmajorgrids,
    xmin=-0.0, xmax=0.5,
    xtick style={color=white!15!black},
    ylabel={Accuracy},
    ymajorgrids,
    ymin=26.973718350705, ymax=33.472489923295,
    ytick style={color=white!15!black},
    ylabel near ticks,
	height=\noisefigheight,
    ]

    \path [draw=color2, fill=color2, opacity=0.2]
    (axis cs:0,32.863074599925)
    --(axis cs:0,32.3462407976)
    --(axis cs:0.05,30.986646844175)
    --(axis cs:0.1,30.690734253225)
    --(axis cs:0.15,30.44085464415)
    --(axis cs:0.2,30.241738361)
    --(axis cs:0.25,29.995148174775)
    --(axis cs:0.3,30.5407432914)
    --(axis cs:0.35,30.8314095492)
    --(axis cs:0.4,30.363533839575)
    --(axis cs:0.45,30.20346797815)
    --(axis cs:0.5,30.0006195131)
    --(axis cs:0.5,30.897035367925)
    --(axis cs:0.5,30.897035367925)
    --(axis cs:0.45,31.08495694525)
    --(axis cs:0.4,31.134814459325)
    --(axis cs:0.35,31.187778891325)
    --(axis cs:0.3,31.265170865775)
    --(axis cs:0.25,30.93118051835)
    --(axis cs:0.2,30.770959538125)
    --(axis cs:0.15,30.9003912823)
    --(axis cs:0.1,31.1338788615)
    --(axis cs:0.05,31.531665142625)
    --(axis cs:0,32.863074599925)
    --cycle;
    
    \path [draw=color3, fill=color3, opacity=0.2]
    (axis cs:0,32.967684883475)
    --(axis cs:0,31.295423451175)
    --(axis cs:0.05,30.106856997925)
    --(axis cs:0.1,29.2344302389)
    --(axis cs:0.15,28.6856863089)
    --(axis cs:0.2,28.39808987305)
    --(axis cs:0.25,28.195182351725)
    --(axis cs:0.3,28.22891662305)
    --(axis cs:0.35,27.829633320925)
    --(axis cs:0.4,27.289688334275)
    --(axis cs:0.45,27.6331881)
    --(axis cs:0.5,28.032070717125)
    --(axis cs:0.5,28.6780799888)
    --(axis cs:0.5,28.6780799888)
    --(axis cs:0.45,28.48289803185)
    --(axis cs:0.4,28.16217566325)
    --(axis cs:0.35,28.571519478375)
    --(axis cs:0.3,29.118154259525)
    --(axis cs:0.25,29.0393842354)
    --(axis cs:0.2,28.928038208025)
    --(axis cs:0.15,29.747818347025)
    --(axis cs:0.1,30.254184573725)
    --(axis cs:0.05,31.30647454215)
    --(axis cs:0,32.967684883475)
    --cycle;

    \addplot [semithick, color2, mark=triangle*, mark options={draw=black}]
    table {%
    0 32.612126924
    0.05 31.252056694
    0.1 30.908223534
    0.15 30.674074174
    0.2 30.498343276
    0.25 30.436348342
    0.3 30.92515297
    0.35 31.006700324
    0.4 30.787095261
    0.45 30.648323062
    0.5 30.408688924
    };

    \addplot [semithick, color3, mark=*, mark options={draw=black}]
    table {%
    0 32.200099946
    0.05 30.725577163
    0.1 29.749159623
    0.15 29.22029648
    0.2 28.658766176
    0.25 28.620138549
    0.3 28.604163742
    0.35 28.234816741
    0.4 27.778678323
    0.45 28.085552979
    0.5 28.346407889
    };
    \end{axis}
    
    \end{tikzpicture}
     
        \caption{\textbf{ogbn-mag}}
        \label{fig:noise-mag}
    \end{subfigure}
    \caption{
    Test accuracy compared to the noise ratio:
    \newgat{} is more robust to structural noise compared to GAT. 
    Each point is an average of 10 runs, error bars show standard deviation.}
    \label{fig:noise}
\end{figure*}

%% file: varmisuse.tex
\subsection{Programs: {\sc{VarMisuse}}}
\label{subsec:varmisuse}
\para{Setup} {\sc{VarMisuse}} \cite{allamanis2018learning} 
is an inductive node-pointing problem that depends on 11 %
types of syntactic and semantic interactions between elements %
in computer programs.

We used the framework of \citet{brockschmidt2019graph}, who performed an extensive hyperparameter tuning
by searching over 30 configurations for every GNN type. 
We took their best GAT hyperparameters and used them to train \newgat{}, without further tuning.

\para{Results} As shown in \Cref{tab:varmisuse-results}, \newgat{} is more accurate than GAT and other GNNs in the SeenProj test sets. Furthermore, \newgat{} achieves an even higher improvement in the \emph{Unseen}Proj test set. 
Overall,
these results demonstrate the power of \newgat{}  in modeling complex relational problems, especially since it outperforms extensively tuned models, without any further 
tuning by us.

\input{varmisuse_table.tex}

%% file: varmisuse_table.tex
\begin{figure}[h!]
    \centering
        \caption{Accuracy (5 runs$\pm$stdev) on {\sc{VarMisuse}}. 
    \newgat{} is more accurate than all GNNs in both test sets, using GAT's hyperparameters. %
    ${\dagger}$ previously reported by \citet{brockschmidt2019graph}.
    }%
    \footnotesize
    \setlength\tabcolsep{5 pt}
        \begin{tabular}{@{}llcc@{}}
        \toprule
         & Model & \multicolumn{1}{c}{SeenProj} & \multicolumn{1}{c}{UnseenProj} \\
         \midrule
		 \multirow{2}{*}{\shortstack[c]{No-\\Attention}}& GCN$^{\dagger}$ & 87.2\err{1.5}&  81.4\err{2.3}\\
		 & GIN$^{\dagger}$ & 87.1\err{0.1}&  81.1\err{0.9}\\
         \midrule
          \multirow{2}{*}{Attention}         
          & GAT$^{\dagger}$	   & 86.9\err{0.7}& 81.2\err{0.9}\\
          & \newgat{}	   & \textbf{88.0}\err{1.1}& \textbf{82.8}\err{1.7}\\
\bottomrule
    \end{tabular}
    \label{tab:varmisuse-results}
\end{figure}

%% file: node-prediction.tex
\subsection{Node-Prediction}
\label{subsec:node}
We further compare  \newgat{}, GAT, and other GNNs on four node-prediction datasets from OGB.

\input{node-ogb-table.tex}

\para{Results}
Results are shown in \Cref{tab:node-ogb}.
In all settings and all datasets, \newgat{} is more accurate than GAT and the non-attentive GNNs. 
Interestingly, in the datasets of \Cref{tab:node-ogbn-results}, \emph{even a single head of \newgat{} outperforms GAT with 8 heads}.
In \Cref{tab:proteins-results} (\textbf{ogbn-proteins}), increasing the number of heads results in a major improvement for GAT (from 70.77 to 78.63), while \newgat{} already gets most of the benefit using a single attention head. 
These results demonstrate the superiority of \newgat{} over GAT in node prediction (and even with a single head), thanks to \newgat{}'s dynamic attention.

%% file: node-ogb-table.tex
\begin{table*}[h!]
    \centering
	\caption{Average accuracy (\Cref{tab:node-ogbn-results}) and ROC-AUC (\Cref{tab:proteins-results}) in node-prediction datasets (10 runs$\pm$std). In all  datasets, \newgat{} outperforms GAT.
	${\dagger}$ -- previously reported by \citet{hu2020open}.
	} 
	\begin{subtable}{0.8\textwidth}
    		\caption{}
	\centering
    \footnotesize
    \begin{tabu}{llccc}
        \toprule
        Model & Attn. Heads     & \textbf{ogbn-arxiv}                           & \textbf{ogbn-products}                & \textbf{ogbn-mag}
        \\ 
        \midrule
        GCN$^{\dagger}$ & 0 & 71.74\err{0.29} & 78.97\err{0.33} & 30.43\err{0.25} \\
        GraphSAGE$^{\dagger}$ & 0 & 71.49\err{0.27} & 78.70\err{0.36} & 31.53\err{0.15} \\
        \midrule
        \multirow{2}{*}{\shortstack[c]{GAT}} 
        & 1      & 71.59\err{0.38}                   & 79.04\err{1.54}          & 32.20\err{1.46}  \\ 
        & 8     & 71.54\err{0.30}                   & 77.23\err{2.37}          & 31.75\err{1.60}  \\ 
        \midrule
        \multirow{2}{*}{\shortstack[c]{\newgat{} (this work)}} 

        & 1  & 71.78\err{0.18}          & \textbf{80.63}\err{0.70} & \textbf{32.61}\err{0.44} \\ 
        & 8   & \textbf{71.87}\err{0.25}          & 78.46\err{2.45}           & 32.52\err{0.39}  \\ 
        \bottomrule
        \end{tabu}
        \label{tab:node-ogbn-results}
	\end{subtable}
	\begin{subtable}{0.18\textwidth}
		\caption{}
    \centering
    \footnotesize
    \begin{tabu}{c}
        \toprule
		\textbf{ogbn-proteins}\\ 
        \midrule
        72.51\vphantom{$^{\dagger}$}\err{0.35}  \\ %
        77.68\vphantom{$^{\dagger}$}\err{0.20} \\ %
        \midrule
		70.77\err{5.79}   \\ 
        78.63\err{1.62} \\ 
        \midrule

        77.23\err{3.32} \\ 
        \textbf{79.52}\err{0.55}      \\
        \bottomrule
        \end{tabu}
    
        \label{tab:proteins-results}
	\end{subtable}

	\label{tab:node-ogb}
\end{table*}

%% file: graph-prediction.tex
\subsection{Graph-Prediction: QM9}
\label{subsec:graph}
\para{Setup} 
In the QM9 dataset \cite{ramakrishnan2014quantum, gilmer2017neural},
each graph is a molecule %
and the goal is to regress each graph to 13 real-valued quantum chemical properties.
We used the implementation of \citet{brockschmidt2019graph} who performed an extensive hyperparameter search over 500 configurations; we took their best-found configuration of GAT to implement \newgat{}.

\input{qm_table.tex}

\para{Results}
\Cref{tab:qm-results} shows the main results: \newgat{} achieves a lower (better) average error than GAT, by 11.5\% relatively. GAT achieves the overall highest average error. 
In some properties, the non-attentive GNNs, GCN and GIN, perform best. We hypothesize that attention is not needed in modeling these properties. 
Generally, %
\newgat{} achieves the lowest overall average relative error (rightmost column).

%% file: qm_table.tex
\newcommand{\qmwidth}{4mm}

\begin{table*}[h!]
    \caption{Average error rates (lower is better), 5 runs for each property, on the QM9 dataset. 
    The best result among GAT and \newgat{} is marked in \textbf{bold}; the globally best result among all GNNs is  marked in \textbf{\uline{bold and underline}}.
        $\dagger$ was previously tuned and reported by \citet{brockschmidt2019graph}.
    }
    \centering
    \footnotesize
    \begin{tabu}{lR{\qmwidth}R{\qmwidth}R{\qmwidth}R{\qmwidth}R{\qmwidth}R{\qmwidth}R{\qmwidth}R{\qmwidth}R{\qmwidth}R{\qmwidth}R{\qmwidth}R{\qmwidth}R{\qmwidth}|r@{}}
        \toprule
		& \multicolumn{13}{c}{Predicted Property} & Rel. to \\
		Model & \multicolumn{1}{c}{1} & \multicolumn{1}{c}{2} & \multicolumn{1}{c}{3} & \multicolumn{1}{c}{4} & \multicolumn{1}{c}{5} & \multicolumn{1}{c}{6} & \multicolumn{1}{c}{7} & \multicolumn{1}{c}{8} & \multicolumn{1}{c}{9} & \multicolumn{1}{c}{10} & \multicolumn{1}{c}{11} & \multicolumn{1}{c}{12} & \multicolumn{1}{c}{13} &  GAT \\
		\midrule
		\normalsize{GCN$^\dagger$} & 3.21 & \textbf{\uline{4.22}} & 1.45 & 1.62 & 2.42 & 16.38 & 17.40 & 7.82 & 8.24 & 9.05 & 7.00 & 3.93 & \textbf{\uline{1.02}} & -1.5\% \\
		\normalsize{GIN$^\dagger$} & \textbf{\uline{2.64}} & 4.67 & 1.42 & 1.50 & \textbf{\uline{2.27}} & \textbf{\uline{15.63}} & \textbf{\uline{12.93}} & \textbf{\uline{5.88}} & 18.71 & \textbf{\uline{5.62}} & \textbf{\uline{5.38}} & \textbf{\uline{3.53}} & 1.05 & -2.3\% \\
		\midrule

		\normalsize{GAT$^\dagger$} & 2.68 & 4.65 & 1.48 & 1.53 & 2.31 & 52.39 & 14.87 & 7.61 & 6.86 & 7.64 & 6.54 & 4.11 & \textbf{1.48} & +0\% \\
		\normalsize{\newgat{}} & \textbf{2.65} & {\textbf{4.28}} & \textbf{\uline{1.41}} & \textbf{\uline{1.47}} & \textbf{2.29} & {\textbf{16.37}} & \textbf{14.03} & \textbf{6.07} & \textbf{\uline{6.28}} & {\textbf{6.60}} & {\textbf{5.97}} & {\textbf{3.57}} & 1.59 & \textbf{\uline{-11.5}\%} \\		
        \bottomrule
    \end{tabu}

    \label{tab:qm-results}
\end{table*} 

%% file: link-prediction.tex
\subsection{Link-Prediction}
\label{subsec:link}
We compare \newgat{}, GAT, and other GNNs in link-prediction datasets from OGB. %

\input{link-tables.tex}

\para{Results}
\Cref{tab:link-results} shows that 
in 
all datasets, \newgat{} achieves a higher MRR than GAT, which achieves the lowest MRR.
However, the non-attentive GraphSAGE performs better than all attentive GNNs. We hypothesize that attention might not be needed in these datasets. Another possibility is that dynamic attention is especially useful in graphs that have \emph{high node degrees}: in \textbf{ogbn-products} and \textbf{ogbn-proteins} (\Cref{tab:node-ogb}) the average node degrees are 50.5 and 597, respectively (see \Cref{tab:stat-ogb} in \Cref{sec:stats}). \textbf{ogbl-collab} and \textbf{ogbl-citation2} (\Cref{tab:link-results}), however, have much lower average node degrees -- of 8.2 and 20.7. We hypothesize that a dynamic attention mechanism is especially useful to select the most relevant neighbors when the total number of neighbors is high. We leave the study of the effect of the datasets's average node degrees on the optimal GNN architecture for future work.

%% file: link-tables.tex
\begin{table*}[h!]
\caption{Average Hits@50 %
(\Cref{tab:collab-results}) and mean reciprocal rank (MRR) (\Cref{tab:citation2-results}) in link-prediction benchmarks from OGB (10 runs$\pm$std). 
The best result among GAT and \newgat{} is marked in \textbf{bold}; the best result among all GNNs is marked in \textbf{\uline{bold and underline}}. 
${\dagger}$ was reported by \citet{hu2020open}. 
}
    \centering
	\begin{subtable}{0.6\textwidth}
    \caption{}
    \centering
    \footnotesize
    \begin{tabu}{llcc}
        \toprule
        \rowfont{\footnotesize}
            & & \multicolumn{2}{c}{\textbf{ogbl-collab}}\\
       Model & Attn. Heads & \multicolumn{1}{c}{w/o val edges}& \multicolumn{1}{c}{w/ val edges} \\
        \midrule
        \multirow{2}{*}{\shortstack[c]{No-\vphantom{$^{\dagger}$}\\Attention}}
        & GCN$^{\dagger}$ & 44.75\err{1.07} & 47.14\err{1.45} \\
        & GraphSAGE$^{\dagger}$ & \textbf{\uline{48.10}}\err{0.81} & \textbf{\uline{54.63}}\err{1.12} \\
        \midrule
        \multirow{2}{*}{\shortstack[c]{GAT}} 
        & \normalsize{GAT$_{1h}$} & 39.32\err{3.26} & 48.10\err{4.80} \\
        & \normalsize{GAT$_{8h}$} & 42.37\err{2.99} & 46.63\err{2.80} \\
        \multirow{2}{*}{\shortstack[c]{\newgat{}}} 
        & \normalsize{\newgat{}$_{1h}$} & 42.00\err{2.40} & 48.02\err{2.77} \\
        & \normalsize{\newgat{}$_{8h}$} & \textbf{42.85}\err{2.64} & \textbf{49.70}\err{3.08} \\
        \bottomrule
    \end{tabu}
    \label{tab:collab-results}
	\end{subtable}
	\begin{subtable}{0.15\textwidth}
        \caption{}
    \centering
    \footnotesize
    \begin{tabular}{c}
        \toprule
        \textbf{ogbl-citation2}\\ 
        \\
        \midrule
        80.04\vphantom{$^{\dagger}$}\err{0.25} \\
        \textbf{\uline{80.44}}\vphantom{$^{\dagger}$}\err{0.10} \\
        \midrule
        79.84\err{0.19} \\ 
        75.95\err{1.31} \\ 
        80.33\err{0.13} \\ 
        \textbf{80.14}\err{0.71} \\ 
        \bottomrule
        \end{tabular}
         \label{tab:citation2-results}
	\end{subtable}

\label{tab:link-results}
\end{table*}

%% file: related.tex
\section{Related Work}
\label{sec:related}

\para{Attention in GNNs}
Modeling pairwise interactions between elements in graph-structured data goes back to interaction networks \cite{battaglia2016interaction,hoshen2017vain} and relational networks \cite{santoro2017simple}.
The GAT formulation of \citet{velic2018graph} rose as the most popular framework for attentional GNNs, thanks to its simplicity, generality, and
applicability beyond reinforcement learning 
\cite{denil2017programmable, duan2017one}.
Nevertheless, in this work, we show that the popular and widespread definition of GAT is severely constrained to static attention only. %

\para{Other graph attention mechanisms}
Many works employed GNNs with attention mechanisms other than the standard GAT's \cite{zhang2018gaan,thekumparampil2018attention,gao2019graph, lukovnikov2021gated, shi2020masked, dwivedi2020generalization, busbridge2019relational,rong2020self,velivckovic2020pointer}, 
and \citet{lee2018attention} conducted an extensive survey of attention types in GNNs.
However, none of these works identified the monotonicity of GAT's attention mechanism, the theoretical differences between attention types, nor empirically compared their performance.
\citet{kim2021how} compared two graph attention mechanisms empirically, but in a specific self-supervised scenario, without observing the theoretical difference in their expressiveness.

\para{The static attention of GAT}
\citet{qiu2018deepinf} recognized the order-preserving property of GAT, but did not identify the severe theoretical constraint that this property implies: the inability to perform dynamic attention (\Cref{theorem:monotonic}). 
Furthermore, they presented GAT's monotonicity as a \emph{desired} trait~(!) %
To the best of our knowledge, our work is the first work to recognize the inability of GAT to perform dynamic attention and its practical harmful consequences.

%% file: conclusion.tex
\section{Conclusion}
\label{sec:conclusion}
In this paper, we identify that the popular and widespread Graph Attention Network does not compute \emph{dynamic} attention. Instead, the attention mechanism in the standard definition and implementations of GAT is only \emph{static}: for any query, its neighbor-scoring is monotonic with respect to per-node scores.
As a result, GAT cannot even express simple alignment problems.
To address this limitation, we introduce a simple fix and propose \newgat{}: by modifying the order of operations in GAT, \newgat{} achieves a universal approximator attention function and is thus strictly more powerful than GAT.

We demonstrate the empirical advantage of \newgat{} over GAT in a synthetic problem that requires~dynamic selection of nodes, and in 11 benchmarks from OGB and other public datasets. Our experiments show that \newgat{} outperforms GAT in all benchmarks while having the same parametric cost.

We encourage the community to use \newgat{} instead of GAT whenever comparing new GNN architectures to the common strong baselines. 
In complex tasks and domains and in challenging datasets, a model that uses GAT as an internal component can replace it with \newgat{} to benefit from a strictly more powerful model.
To this end, we make our code publicly available at \url{https://github.com/tech-srl/how_attentive_are_gats} , and \newgat{} is available as part of the PyTorch Geometric library, the Deep Graph Library, and TensorFlow GNN. 
An annotated implementation is available at \url{https://nn.labml.ai/graphs/gatv2/} .

%% file: acknowledgements.tex
\section*{Acknowledgments}

We thank Gail Weiss for the helpful discussions, thorough feedback, and inspirational paper \cite{weiss2018practical}. 
We also thank Petar Veli{\v{c}}kovi{\'c} for the useful discussion about the complexity and implementation of GAT.

%% file: appendix.tex
\section{Proof for \Cref{theorem:dynamic}}
\label{sec:proof-dynamic}
For brevity, we repeat our definition of dynamic attention (\Cref{def:dynamic}):

\paragraph{Definition \getrefnumber{def:dynamic}}(Dynamic attention).

\input{def-dynamic.tex}
\addtocounter{theorem}{-1}

\begin{theorem}
	\input{theorem-gat2.tex}
\end{theorem}

\begin{proof} 
	Let $\mathcal{G}=\left(\mathcal{V},\mathcal{E}\right)$ be a graph modeled by a \newgat{} layer, having node representations $\set{\vh_1,..., \vh_n}$, and
let $\varphi : \intn \rightarrow \intn$ be any node mapping $\intn \rightarrow \intn$. %
	We define $g: \mathbb{R}^{2d} \rightarrow \mathbb{R}$ as follows:
	\begin{equation}
		g\left(\vx\right) = 
		\begin{cases}
			1 & \exists i: \vx=\left[\vh_i \|\vh_{\varphi\left(i\right)}\right] \\
			0 & \textrm{otherwise}
		\end{cases}
	\end{equation}

	Next, we define a \emph{continues} function $\widetilde{g}: \mathbb{R}^{2d} \rightarrow \mathbb{R}$ that equals to $g$ in only specific $n^2$ inputs:
	\begin{equation}
		\widetilde{g}(\left[\vh_i \| \vh_j\right]) = g(\left[\vh_i \| \vh_j\right]), \forall i,j\in\intn
	\end{equation}
	For all other inputs $x\in\mathbb{R}^{2d}$, $\widetilde{g}(x)$ realizes to any values that maintain the continuity of $\widetilde{g}$ (this is possible because we fixed the values of $\widetilde{g}$ for only a finite set of points). 
\footnote{
The function $\widetilde{g}$ is a function that we define for the ease of proof, because the universal approximation theorem is defined for continuous functions, and we only need the scoring function of \newgat{} $e$ to approximate the mapping $\varphi$ in a finite set of points. So, we need the attention function $e$ to approximate $g$ (from Equation 8) in some specific points. But, since $g$ is not continuous, we define $\widetilde{g}$ and use the universal approximation theorem for $\widetilde{g}$. Since $e$ approximates $\widetilde{g}$, $e$ also approximates $g$ in our specific points, as a special case.
We only require that $\widetilde{g}$ will be identical to $g$ in specific $n^2$ points $\set{ \left[h_i \| h_j\right] \mid i,j\in \intn}$. For the rest of the input space, we don't have any requirement on the value of $\widetilde{g}$, except for maintaining the continuity of $\widetilde{g}$. 
There exist infinitely many such possible $\widetilde{g}$ for every given set of keys, queries and a mapping $\varphi$, but the concrete functions are not needed for the proof. } 

Thus, for every node $i\in\mathcal{V}$ and  $j_{\neq \varphi\left(i\right)} \in \mathcal{V}$: %
\begin{equation}
1=\widetilde{g}\left(\left[\vh_i \|\vh_{\varphi\left(i\right)}\right]\right) > \widetilde{g}\left(\left[\vh_i \|\vh_{j}\right]\right)=0	
\end{equation}

If we concatenate the two input vectors, and define the scoring function $e$ of \newgat{} (\Cref{eq:gat2-vs-gat}) as a function of the concatenated vector $\left[\vh_{i} \| \vh_{j}\right]$,
from the universal approximation theorem \cite{hornik1989multilayer, cybenko1989approximation, funahashi1989approximate,hornik1991approximation}, $e$  can approximate $\widetilde{g}$ for any compact subset of $\mathbb{R}^{2d}$. 

Thus, for any sufficiently small $\epsilon$ (any $0<\epsilon<\nicefrac{1}{2}$) there exist parameters $\mW$ and $\va$ such that
for every node $i\in\mathcal{V}$ and every $j_{\ne \varphi\left(i\right)}$:
\begin{equation}
	e_{\mW,\va}\left(\vh_i, \vh_{\varphi\left(i\right)}\right) > 1 - \epsilon > 0 + \epsilon >
	  e_{\mW,\va}\left(\vh_i, \vh_{j}\right)
	\end{equation}
and due to the increasing monotonicity of $\mathrm{softmax}$:
\begin{equation}
	\alpha_{i,\varphi\left(i\right)}  > \alpha_{i,j} 
\end{equation}
\end{proof}

\paragraph{The choice of nonlinearity}
In general, these results hold if \newgat{} had used any common non-polynomial activation function (such as ReLU, sigmoid, or the hyperbolic tangent function).
The LeakyReLU activation function of \newgat{} does not change its universal approximation ability
\cite{leshno1993multilayer,pinkus1999approximation,park2021minimum}, and it was chosen only for consistency with the original definition of GAT.

\section{Training details}
\label{sec:training}
In this section we elaborate on the training details of all of our experiments.
All models use residual connections as in \citet{velic2018graph}. 
All used code and data are publicly available under the MIT license.

\subsection{Node- and Link-Prediction}
We used the provided splits of OGB \cite{hu2020open} and the Adam optimizer. We tuned the following hyperparameters: number of layers $\in\{2,3,6\}$, hidden size $\in\{64, 128, 256\}$, learning rate $\in\{0.0005, 0.001, 0.005, 0.01\}$ and sampling method -- full batch, GraphSAINT \citep{zeng2019graphsaint} and NeighborSampling \citep{hamilton2017inductive}. We tuned hyperparameters according to validation score and early stopping. 
The final hyperparameters are detailed in \Cref{tab:training-details}.
\input{training-details-table.tex}

\subsection{Robustness to Noise}
In these experiments, we used the same best-found hyperparameters in node-prediction, 
with 8 attention heads in \textbf{ogbn-arxiv} and 1 head in \textbf{ogbn-mag}. 
Each point is an average of 10 runs.

\subsection{Synthetic Benchmark: \kchoose{}}
In all experiments, we used a learning rate decay of $0.5$, a hidden size of $d=128$, a batch size of $1024$, and the Adam optimizer. 

We created a separate dataset for every graph size ($k$), 
and we split each such dataset to train and test with a ratio of 80:20. 
Since this is a contrived problem, we did not use a validation set, and the reported test results can be thought of as validation results.
Every model was trained on a fixed value of $k$.
Every key node (bottom row in \Cref{fig:kchoose}) was encoded as a sum of learned attribute embedding and a value embedding, followed by ReLU.

We experimented with layer normalization, batch normalization, dropout, various activation functions and various learning rates. None of these changed the general trend, so the experiments in \Cref{fig:kchoose-results} were conducted without any normalization, without dropout and a learning rate of $0.001$.

\subsection{Programs: {\sc{VarMisuse}}}
We used the code, splits, and the same best-found configurations as \citet{brockschmidt2019graph}, who performed an extensive hyperparameter tuning by searching over 30 configurations for each GNN type. We trained each model five times.

We took the best-found hyperparameters of \citet{brockschmidt2019graph} for GAT and used them to train \newgat{}, without any further tuning.

\subsection{Graph-Prediction: QM9}
\label{se:qm9-details}
We used the code and splits of \citet{brockschmidt2019graph} who performed an extensive hyperparameter search over 500 configurations.
We took the best-found hyperparameters of \citet{brockschmidt2019graph} for GAT and used them to train \newgat{}.
The only minor change from GAT is placing a residual connection after every layer, rather than after every other layer,
which is within the experimented hyperparameter search that was reported by \citet{brockschmidt2019graph}.

\subsection{Compute and Resources}
\label{subsec:compute}
Our experiments consumed approximately 100 days of GPU in total. We used cloud GPUs of type V100, and we used RTX 3080 and 3090 in local GPU machines.

 \section{Data Statistics}
 \label{sec:stats}

 \subsection{Node- and Link-Prediction Datasets}
 Statistics of the OGB datasets we used for node- and link-prediction are shown in \Cref{tab:stat-ogb}.

 \begin{table*}[h!]
	\centering
	\begin{tabular}{lrrrr}
	  \toprule
	  	Dataset &  \# nodes & \# edges & Avg. node degree & Diameter  \\
	  \midrule
		\textbf{ogbn-arxiv}& 169,343 & 1,166,243 & 13.7 & 23 \\
		\textbf{ogbn-mag} & 1,939,743 & 21,111,007 & 21.7 & 6  \\
		\textbf{ogbn-products} & 2,449,029 & 61,859,140 & 50.5 & 27 \\
		\textbf{ogbn-proteins}  & 132,534  & 39,561,252 & 597.0 & 9 \\
	  \midrule
		\textbf{ogbl-collab}  & 235,868 & 1,285,465 & 8.2 & 22 \\
		\textbf{ogbl-citation2} & 2,927,963 & 30,561,187 & 20.7 & 21 \\
	  \bottomrule
	\end{tabular}
	\caption{ Statistics of the OGB datasets \cite{hu2020open}.}
	\label{tab:stat-ogb}
  \end{table*}

\subsection{QM9}
Statistics of the QM9 dataset, as used in \citet{brockschmidt2019graph} are shown in \Cref{tab:stat-qm}.
\begin{table*}[h!]
    \centering
       \begin{tabular}{lrrrr}
        \toprule
         &  Training & Validation & Test \\
         \midrule
         \# examples & 110,462 & 10,000 & 10,000 \\
         \# nodes - average & 18.03 & 18.06 & 18.09\\
         \# edges - average  & 18.65 & 18.67 & 18.72 \\
       Diameter - average & 6.35 & 6.35 & 6.35 \\
	    \bottomrule
    \end{tabular}
	\caption{Statistics of the QM9 chemical dataset \cite{ramakrishnan2014quantum} as used by \citet{brockschmidt2019graph}.}
    \label{tab:stat-qm}
\end{table*} 

\subsection{{\sc{VarMisuse}}}
Statistics of the {\sc{VarMisuse}} dataset, as used in \citet{allamanis2018learning} and \citet{brockschmidt2019graph}, are shown in \Cref{tab:stat-varmisuse}.
\begin{table*}[h!]
    \centering
          \begin{tabular}{lrrrr}
        \toprule
         &  Training & Validation & UnseenProject Test & SeenProject Test \\
         \midrule
         \# graphs & 254360 & 42654 & 117036 & 59974  \\
         \# nodes - average & 2377 & 1742 & 1959 & 3986 \\
         \# edges - average  & 7298 & 7851 & 5882 & 12925 \\
    	    Diameter - average & 7.88 & 7.88 & 7.78 & 7.82 \\
	    \bottomrule
    \end{tabular}
	\caption{Statistics of the {\sc{VarMisuse}} dataset \cite{allamanis2018learning} as used by \citet{brockschmidt2019graph}.}
    \label{tab:stat-varmisuse}
\end{table*} 

\input{additional_results.tex}

\section{Additional Comparison with Transformer-style Attention (\dpgat{})}
\label{sec:dpgat}

The main goal of our paper is to highlight a severe theoretical limitation of the highly popular GAT architecture, and propose a minimal fix.

We perform additional empirical comparison to \dpgat{}, which follows \citet{luong15} and the dot-product attention of 
the Transformer 
\cite{vaswani2017attention}.
We define \dpgat{} as:
\begin{align}
	&			 \mathrm{\dpgat{}} \text{ \cite{vaswani2017attention}:} &
	e\left(\vh_i, \vh_j\right) =
		\left(\left(\vh_i^{\top}\mQ\right) \cdot \left(\vh_j^{\top}\mK \right)^{\top} \right)
		/ \sqrt{d_k} 
		\label{eq:dpgat}
	\end{align}
Variants of \dpgat{} were used in prior work \cite{gao2019graph,  dwivedi2020generalization,rong2020self,velivckovic2020pointer,kim2021how}, and we consider it here for the conceptual and empirical comparison with GAT.

Despite its popularity, \dpgat{} is \emph{strictly weaker} than \newgat{}. 
\dpgat{} provably performs dynamic attention for any set of node representations only if they are \emph{linearly independent} (see \Cref{theorem:dpgat} and its proof in \Cref{sec:dpgat-proof}).
Otherwise, there are examples of node representations that \emph{are} linearly dependent and mappings $\varphi$, 
for which dynamic attention does not hold (\Cref{subsec:weaker}).
This constraint is not harmful when violated in practice,
because every node has only a small set of neighbors, rather than all possible nodes in the graph; further, some nodes possibly never need to be ``selected'' in practice.

\subsection{Proof that \dpgat{} Performs Dynamic Attention for Linearly Independent Node Representations}
\label{sec:dpgat-proof}
\begin{theorem}
	A \dpgat{} layer computes dynamic attention
	for any set of node representations $\mathbb{K}=\mathbb{Q}=\set{\vh_1,...,\vh_n}$ that are linearly independent.
	\label{theorem:dpgat}
\end{theorem}
\begin{proof}
	Let $\mathcal{G}=\left(\mathcal{V},\mathcal{E}\right)$ be a graph modeled by a \dpgat{} layer, having linearly independent node representations $\set{\vh_1,..., \vh_n}$. Let $\varphi : \intn \rightarrow \intn$ be any node mapping $\intn \rightarrow \intn$.
	
	We denote the i$^{th}$ row of a matrix $\mM$ as $\mM_i$.
	
	We define a matrix $\mP$ as:
	\begin{equation}
		\mP_{i,j} = 
		\begin{cases}
			1 & j = \varphi(i) \\
			0 & \textrm{otherwise}
		\end{cases}
		\label{eq:P}
	\end{equation}

	Let $\mX\in\mathbb{R}^{n}\times \mathbb{R}^{d}$ be the matrix holding the graph's node representations as its rows:

	\begin{align}
		\mX &= \begin{bmatrix}
			\text{---} & \vh_1 & \text{---} \\
			\text{---} & \vh_2 & \text{---} \\
			& \vdots & \\
			\text{---} & \vh_n & \text{---}
			 \end{bmatrix}
			 \label{eq:X}
	\end{align}
	
	Since the rows of $\mX$ are linearly independent, it necessarily holds that $d \ge n$.

	Next, we find weight matrices $\mQ\in\mathbb{R}^{d}\times \mathbb{R}^{d}$ and $\mK \in \mathbb{R}^{d}\times \mathbb{R}^{d}$ such that:
	\begin{equation}
		(\mX\mQ)\cdot(\mX\mK)^{\top}=\mP
		\label{eq:QK=P}
	\end{equation}
	To satisfy \Cref{eq:QK=P}, we choose $\mQ$ and $\mK$ such that $\mX\mQ=\mU$ and $\mX\mK=\mP^{\top}\mU$ where $\mU$ is an orthonormal matrix ($\mU \cdot \mU^{\top} = \mU^{\top} \cdot \mU = I$).

We can obtain $\mU$ using the singular value decomposition (SVD) of $\mX$:
	\begin{equation}
		\mX = \mU\bm{\Sigma}\mV^{\top}
		\label{eq:X-SVD}
	\end{equation}
	Since $\bm{\Sigma}\in\mathbb{R}^{n}\times \mathbb{R}^{n}$ and $\mX$ has a full rank, $\bm{\Sigma}$ is invertible, and thus:
	\begin{equation}
		\mX\mV\bm{\Sigma}^{-1}= \mU
	\end{equation}
	Now, we define $\mQ$ as follows:
	\begin{equation}
		\mQ=\mV\bm{\Sigma}^{-1}
		\label{eq:Q}
	\end{equation}
	Note that $\mX\mQ=\mU$, as desired.

	To find $\mK$ that satisfies $\mX\mK=\mP^{\top}\mU$, we use \Cref{eq:X-SVD} and require:
	\begin{equation}
		\mU\bm{\Sigma}\mV^{\top}\mK= \mP^{\top}\mU
	\end{equation}
	and thus:
	\begin{equation}
		\mK= \mV\bm{\Sigma}^{-1}\mU^T\mP^{\top}\mU
		\label{eq:K}
	\end{equation}

	We define:
	\begin{equation}
		z\left(\vh_i,\vh_j\right) = e\left(\vh_i, \vh_j\right) \cdot \sqrt{d_k}
	\end{equation}
	Where $e$ is the attention score function of \dpgat{} (\Cref{eq:dpgat}).

	Now, for a query $i$ and a key $j$, and the corresponding representations $\vh_i,\vh_j$:
	\begin{align}
		z\left(\vh_i,\vh_j\right) & =
	 \left(\vh_i^{\top}\mQ\right) \cdot \left(\vh_j^{\top}\mK\right)^{\top} \\
	& = \left(\mX_i\mQ\right) \cdot \left(\mX_j\mK\right)^{\top}
	\end{align} 

	Since $\mX_i\mQ  = \left(\mX\mQ\right)_i$ and $\mX_j\mK  = \left(\mX\mK\right)_j$, we get
	\begin{equation}
		z\left(\vh_i,\vh_j\right) =  \left(\mX\mQ\right)_i \cdot \left(\left(\mX\mK\right)_j\right)^{\top}= \mP_{i,j}
	\end{equation}
	Therefore:
	\begin{equation}
		z\left(\vh_i,\vh_j\right) = 
		\begin{cases}
			1 & j = \varphi(i)\\
			0 & \textrm{otherwise}
		\end{cases}
	\end{equation}

	And thus:
	\begin{equation}
		e\left(\vh_i,\vh_j\right) = 
		\begin{cases}
			1 / \sqrt{d_k} & j = \varphi(i) \\
			0 & otherwise
		\end{cases}
	\end{equation}

	To conclude, 
	for every selected query $i$ and any key $j_{\ne \varphi\left(i\right)}$:
	\begin{equation}
		e\left(\vh_i, \vh_{\varphi(i)}\right) > e\left(\vh_i, \vh_{j}\right)
	\end{equation}
	and due to the increasing monotonicity of $\mathrm{softmax}$:
	\begin{equation}
		\alpha_{i,\varphi\left(i\right)}  > \alpha_{i,j} 
	\end{equation}
	Hence, a \dpgat{} layer computes dynamic attention.

	In the case that $d > n$, we apply SVD to the full-rank matrix $\mX\mX^{\top}\in \mathbb{R}^{n\times n}$, and follow the same steps to construct $\mQ$ and $\mK$.
	
	In the case that 
	$\mQ\in\mathbb{R}^{d}\times \mathbb{R}^{d_{k}}$ and $\mK \in \mathbb{R}^{d}\times \mathbb{R}^{d_{k}}$
	and $d_{k} > d$, we can use the same $\mQ$ and $\mK$ (\Cref{eq:Q,eq:K}) padded with zeros. 
	We define the $\mQ'\in\mathbb{R}^{d}\times \mathbb{R}^{d_{key}}$ and $\mK'\in\mathbb{R}^{d}\times \mathbb{R}^{d_{key}}$ as follows:
	\begin{align}
		\mQ'_{i,j} & = 
		\begin{cases}
			\mQ_{i,j} & j \le d\\
			0 & \textrm{otherwise}
		\end{cases} \\ 
		\mK'_{i,j} & = 
		\begin{cases}
			\mK_{i,j} & j \le d\\
			0 & \textrm{otherwise}
		\end{cases}
	\end{align}
\end{proof}

\subsection{\dpgat{} is strictly weaker than \newgat{}}
\label{subsec:weaker}
There are examples of node representations that are linearly dependent and mappings $\varphi$, for which dynamic attention does not hold. First, we show a simple 2-dimensional example, and then we show the general case of such examples.

\input{dpgat_example_figure.tex}

Consider the following linearly dependent set of vectors $\mathbb{K}=\mathbb{Q}$ (\Cref{fig:dpgat-counterexample}):
\begin{align}
	\vh_0 & = \hat{\vx} \\
	\vh_1 & = \hat{\vx} + \hat{\vy} \\
	\vh_2 & = \hat{\vx} + 2\hat{\vy}
\end{align}
where $\hat{\vx}$ and $\hat{\vy}$ are the cartesian unit vectors.
We define $\beta \in \set{0,1,2}$ 
to express $\set{\vh_0,\vh_1,\vh_2}$ using the same expression:
\begin{equation}
	\vh_\beta = \hat{\vx} + \beta\hat{\vy}
\end{equation}
Let $\vq \in \mathbb{Q}$ be any query vector. For brevity, we define the unscaled dot-product attention score as $s$:
\begin{equation}
	s\left(\vq,\vh_\beta\right) =
	e\left(\vq, \vh_\beta\right) \cdot \sqrt{d_k}
\end{equation}
Where $e$ is the attention score function of \dpgat{} (\Cref{eq:dpgat}).
The (unscaled) attention score between $\vq$ and $\set{\vh_0,\vh_1,\vh_2}$ is:
\begin{align}
	s\left(\vq,\vh_\beta\right) & = \left(\vq^{\top}\mQ\right) \left(\vh_\beta^{\top}\mK\right)^{\top} \\
	& = \left(\vq^{\top}\mQ\right) \left(\left(\hat{\vx} + \beta\hat{\vy}\right)^{\top}\mK\right)^{\top}  \\
	& = \left(\vq^{\top}\mQ\right) \left(\hat{\vx}^{\top}\mK + \beta\hat{\vy}^{\top}\mK\right)^{\top}  \\
	& = \left(\vq^{\top}\mQ\right) \left(\hat{\vx}^{\top}\mK\right)^{\top} + \beta\left(\vq^{\top}\mQ\right)\left(\hat{\vy}^{\top}\mK\right)^{\top}
\end{align}

The first term $\left(\vq^{\top}\mQ\right) \left(\hat{\vx}^{\top}\mK\right)^{\top}$ is unconditioned on $\beta$, and thus shared for every $\vh_\beta$.
Let us focus on the second term $\beta\left(\vq^{\top}\mQ\right)\left(\hat{\vy}^{\top}\mK\right)^{\top}$. If $\left(\vq^{\top}\mQ\right)\left(\hat{\vy}^{\top}\mK\right)^{\top} > 0$, then:
\begin{equation}
	e\left(\vq, \vh_{2}\right) > e\left(\vq, \vh_{1}\right)
\end{equation}

Otherwise, if $\left(\vq^{\top}\mQ\right)\left(\hat{\vy}^{\top}\mK\right)^{\top} \le 0$:
\begin{equation}
	e\left(\vq, \vh_{0}\right) \ge e\left(\vq, \vh_{1}\right)
\end{equation}
Thus, for any query $\vq$, the key $\vh_{1}$ can never get the highest score, and thus cannot be ``selected''. That is, the key $\vh_{1}$ cannot satisfy that $e\left(\vq,\vh_{1}\right)$ is strictly greater than any other key.

In the general case, let $\vh_0,\vh_1\in\mathbb{R}^{d}$ be some non-zero vectors , and $\lambda$ is some scalar such that $0 < \lambda < 1$.

Consider the following linearly dependent set of vectors:
\begin{equation}
	\mathbb{K}=\mathbb{Q}=\set{\beta \vh_1 + \left(1-\beta\right) \vh_0 \mid \beta \in \set{0,\lambda,1}}
\end{equation}

For any query $\vq \in \mathbb{Q}$ and $\beta \in \set{0,\lambda,1}$ we define:
\begin{equation}
	s\left(\vq,\beta\right) =
	e\left(\vq, \left(\beta \vh_1 + \left(1-\beta\right) \vh_0\right)\right) \cdot \sqrt{d_k}
\end{equation}
Where $e$ is the attention score function of \dpgat{} (\Cref{eq:dpgat}).

Therefore:
\begin{align}
	s\left(\vq,\beta\right) & = \left(\vq^{\top}\mQ\right)  \left(\left(\beta  \vh_1 + \left(1-\beta\right) \vh_0\right)^{\top}\mK\right)^{\top} \\
	& = \left(\vq^{\top}\mQ\right)  \left(\beta \vh_1^{\top}\mK + \left(1-\beta\right) \vh_0^{\top}\mK\right)^{\top} \\
	& = \left(\vq^{\top}\mQ\right)  \left(\beta \vh_1^{\top}\mK + \vh_0^{\top}\mK -\beta \vh_0^{\top}\mK\right)^{\top} \\
	& = \left(\vq^{\top}\mQ\right)  \left(\beta \left(\vh_1^{\top}\mK -\vh_0^{\top}\mK\right) + \vh_0^{\top}\mK\right)^{\top} \\
	& = \beta \left(\vq^{\top}\mQ\right) \left(\vh_1^{\top}\mK -\vh_0^{\top}\mK\right)^{\top} + \left(\vq^{\top}\mQ\right) \left(\vh_0^{\top}\mK \right)^{\top}
\end{align}

If $\left(\vq^{\top}\mQ\right) \left(\vh_1^{\top}\mK -\vh_0^{\top}\mK\right)^{\top} > 0$:
\begin{equation}
	e\left(\vq, \vh_{1}\right) > e\left(\vq, \vh_{\lambda}\right)
\end{equation}

Otherwise, if $\left(\vq^{\top}\mQ\right) \left(\vh_1^{\top}\mK -\vh_0^{\top}\mK\right)^{\top} \le 0$:
\begin{equation}
	e\left(\vq, \vh_{0}\right) \ge e\left(\vq, \vh_{\lambda}\right)
\end{equation}
Thus, for any query $\vq$, the key $\vh_{\lambda}$ cannot be selected. That is, the key $\vh_{\lambda}$ cannot satisfy that $e\left(\vq,\vh_{\lambda}\right)$ is strictly greater than any other key. Therefore, there are mappings $\varphi$, 
for which dynamic attention does not hold. 

While we prove that \newgat{} computes dynamic attention (\Cref{sec:proof-dynamic}) for \emph{any} set of node representations $\mathbb{K}=\mathbb{Q}$, there are sets of node representations 
and mappings $\varphi$ for which
dynamic attention does not hold for \dpgat{}. Thus, \dpgat{} is strictly weaker than \newgat{}.

\subsection{Empirical Evaluation}
Here we repeat the experiments of \Cref{sec:eval} with \dpgat{}. We remind that \dpgat{} is \emph{strictly weaker} than our proposed \newgat{} (see a proof in \Cref{sec:dpgat-proof}).

\input{all_tables_with_dpgat.tex}

\input{all_tables_with_pvalue.tex}

\input{complexity.tex}

%% file: training-details-table.tex
\begin{table*}[h!]
    \centering
    \small
    \begin{tabular}{lcclc}
    \toprule
    Dataset             & \# layers & Hidden size & Learning rate & Sampling method \\ 
    \midrule
    \textbf{ogbn-arxiv} & 3 & 256 & 0.01 & GraphSAINT \\
    \textbf{ogbn-products} & 3 & 128 & 0.001 & NeighborSampling \\
    \textbf{ogbn-mag} & 2 & 256 & 0.01 & NeighborSampling \\
    \textbf{ogbn-proteins} & 6 & 64 & 0.01  & NeighborSampling \\
    \midrule
    \textbf{ogbl-collab} & 3 & 64 & 0.001 & Full Batch \\
    \textbf{ogbl-citation2} & 3 & 256 & 0.0005 &  NeighborSampling \\
    \bottomrule
    \end{tabular}
    \caption{Training details of node- and link-prediction datasets.}
 
    \label{tab:training-details}
\end{table*}

%% file: additional_results.tex
\section{Additional Results}
\label{sec:additional}

\subsection{\kchoose{}}
\label{sec:additional-kchoose}
\Cref{fig:kchoose-results-additional} shows additional comparison between \newgat{} and GIN \cite{xu2018powerful} in the \kchoose{} problem. \newgat{} easily achieves 100\% train and test accuracy even for $k$$=$$100$ and using only a single head. GIN, although considered as more expressive than other GNNs, cannot perfectly fit the training data (with a model size of $d=128$) starting from $k$$=$$20$.

\input{additional_kchoose_figure.tex}

\subsection{QM9}
\label{sec:additional-qm}
Standard deviation for the QM9 results of \Cref{subsec:graph} are presented in \Cref{tab:qm-results-appendix}.

\renewcommand{\qmwidth}{12mm}

\begin{table*}[h!]
		\centering
		\footnotesize
		\begin{tabu}{lR{\qmwidth}R{\qmwidth}R{\qmwidth}R{\qmwidth}R{\qmwidth}R{\qmwidth}R{\qmwidth}}
			\toprule
			& \multicolumn{6}{c}{Predicted Property} & \\
			Model & \multicolumn{1}{c}{1} & \multicolumn{1}{c}{2} & \multicolumn{1}{c}{3} & \multicolumn{1}{c}{4} & \multicolumn{1}{c}{5} & \multicolumn{1}{c}{6} & \multicolumn{1}{c}{7} \\
			\midrule
			\normalsize{GCN$^\dagger$} & 3.21\err{0.06} & \textbf{\uline{4.22}}\err{0.45} & 1.45\err{0.01} & 1.62\err{0.04} & 2.42\err{0.14} & 16.38\err{0.49} & 17.40\err{3.56}\\
			\normalsize{GIN$^\dagger$} & \textbf{\uline{2.64}}\err{0.11} & 4.67\err{0.52} & 1.42\err{0.01} & 1.50\err{0.09} & \textbf{\uline{2.27}}\err{0.09} & \textbf{\uline{15.63}}\err{1.40} & \textbf{\uline{12.93}}\err{1.81}  \\
			\midrule
			\normalsize{GAT$_{1h}$} & 3.08\err{0.08} & 7.82\err{1.42} & 1.79\err{0.10} & 3.96\err{1.51} & 3.58\err{1.03} & 35.43\err{29.9} & 116.5\err{10.65} \\		
			\normalsize{GAT$_{8h}$$^\dagger$} & 2.68\err{0.06} & 4.65\err{0.44} & 1.48\err{0.03} & 1.53\err{0.07} & 2.31\err{0.06} & 52.39\err{42.58} & 14.87\err{2.88} \\
			\midrule
			\normalsize{\newgat{}$_{1h}$} & 3.04\err{0.06} & 6.38\err{0.62} & 1.68\err{0.04} & 2.18\err{0.61} & 2.82\err{0.25} & 20.56\err{0.70} & 77.13\err{37.93}  \\
			\normalsize{\newgat{}$_{8h}$} & \textbf{2.65}\err{0.05} & \textbf{4.28}\err{0.27} & \textbf{\uline{1.41}}\err{0.04} & \textbf{\uline{1.47}}\err{0.03} & \textbf{2.29}\err{0.15} & \textbf{16.37}\err{0.97} & \textbf{14.03}\err{1.39}  \\		
			\bottomrule
			\rule{0pt}{1ex}    
		\end{tabu}
			\begin{tabu}{lR{\qmwidth}R{\qmwidth}R{\qmwidth}R{\qmwidth}R{\qmwidth}R{\qmwidth}R{\qmwidth}}
			\toprule
			& \multicolumn{6}{c}{Predicted Property} & Rel. to \\
			Model & \multicolumn{1}{c}{8} & \multicolumn{1}{c}{9} & \multicolumn{1}{c}{10} & \multicolumn{1}{c}{11} & \multicolumn{1}{c}{12} & \multicolumn{1}{c}{13} &  GAT$_{8h}$ \\
			\midrule
			\normalsize{GCN$^\dagger$} & 7.82\err{0.80} & 8.24\err{1.25} & 9.05\err{1.21} & 7.00\err{1.51} & 3.93\err{0.48} & \textbf{\uline{1.02}}\err{0.05} & -1.5\% \\
			\normalsize{GIN$^\dagger$} & \textbf{\uline{5.88}}\err{1.01} & 18.71\err{23.36} & \textbf{\uline{5.62}}\err{0.81} & \textbf{\uline{5.38}}\err{0.75} & \textbf{\uline{3.53}}\err{0.37} & 1.05\err{0.11} & -2.3\% \\
			\midrule
			\normalsize{GAT$_{1h}$} & 28.10\err{16.45} & 20.80\err{13.40} & 15.80\err{5.87} & 10.80\err{2.18} & 5.37\err{0.26} & 3.11\err{0.14} & +134.1\% \\		
 		\normalsize{GAT$_{8h}$$^\dagger$} & 7.61\err{0.46} & 6.86\err{0.53} & 7.64\err{0.92} & 6.54\err{0.36} & 4.11\err{0.27} & \textbf{1.48}\err{0.87} & +0\% \\
			\midrule

			\normalsize{\newgat{}$_{1h}$} & 10.19\err{0.63} & 22.56\err{17.46} & 15.04\err{4.58} & 22.94\err{17.34} & 5.23\err{0.36} & 2.46\err{0.65} & +91.6\%\\
			\normalsize{\newgat{}$_{8h}$} & \textbf{6.07}\err{0.77}& \textbf{\uline{6.28}}\err{0.83} & \textbf{6.60}\err{0.79} & \textbf{5.97}\err{0.94} & \textbf{3.57}\err{0.36} & 1.59\err{0.96} & \textbf{\uline{-11.5}}\% \\		
			\bottomrule
		\end{tabu}
	\caption{Average error rates (lower is better), 5 runs $\pm$ standard deviation for each property, on the QM9 dataset. 
	The best result among GAT and \newgat{} is marked in \textbf{bold}; the globally best result among all GNNs is  marked in \textbf{\uline{bold and underline}}.	
	$\dagger$ was previously tuned and reported by \citet{brockschmidt2019graph}.}
	\label{tab:qm-results-appendix}
	\end{table*}

\input{pubmed_table.tex}

%% file: additional_kchoose_figure.tex
\definecolor{ao}{rgb}{0.0, 0.5, 0.0}
\definecolor{mypurple}{HTML}{AB30C4}

\begin{figure*}
    \centering
	\begin{tikzpicture}[scale=1]
	\begin{axis}[
		xlabel={$k$ (number of different keys in each graph)},
		ylabel={\footnotesize{Accuracy}},
		ylabel near ticks,
        legend style={at={(1,0.35)},anchor=west,mark size=2pt,
        	font=\footnotesize, %
        	inner xsep=0pt, inner sep=1pt},
        legend cell align={left},
        xmin=5, xmax=72,
        ymin=-0.0, ymax=105,
        xtick={10,20,...,70},
        ytick={0,10,20,...,100},
        label style={font=\footnotesize},
        ylabel style={font=\footnotesize},
        ylabel shift={-5pt},        
        tick label style={font=\scriptsize} ,
        grid = major,
        major grid style={dotted,gray},
        width = 0.78\linewidth, height = 52.55mm
    ]

\addplot[color=ao, densely dashed, mark options={solid, fill=ao, draw=black}, mark=triangle*, line width=0.5pt, mark size=2pt, visualization depends on=\thisrow{alignment} \as \alignment, nodes near coords, point meta=explicit symbolic,
    every node near coord/.style={anchor=\alignment, font=\scriptsize}] 
table [meta index=2]  {
x   y       label   alignment
10	100		{}		-129
20	100		{}		-90
30	100		{}		-90
40	100		{}		-60
50	100		{}		180
60	100		{}		180
70	100		{}		180
	};
    \addlegendentry{\newgat{}$_{1h}$ train}

\addplot[color=ao, solid, mark options={solid, fill=ao, draw=black}, mark=triangle*, line width=0.5pt, mark size=2pt, visualization depends on=\thisrow{alignment} \as \alignment, nodes near coords, point meta=explicit symbolic,
    every node near coord/.style={anchor=\alignment, font=\scriptsize}] 
table [meta index=2]  {
x   y       label   alignment
10	100		{}		-129
20	100		{}		-90
30	100		{}		-90
40	100		{}		-60
50	100		{}		180
60	100		{}		180
70	100		{}		180
	};
    \addlegendentry{\newgat{}$_{1h}$ test}

\addplot[color=mypurple,densely dashed, mark options={solid, fill=mypurple, draw=black}, mark=*, line width=0.5pt, mark size=2pt, visualization depends on=\thisrow{alignment} \as \alignment, nodes near coords, point meta=explicit symbolic,
    every node near coord/.style={anchor=\alignment, font=\scriptsize}] 
table [meta index=2]  {
x   y       label   alignment
10	100		{}		-129
20	99		{}		-90
30	45		{}		-90
40	12		{}		-90
50	9		{}		-90
60	3		{}		-90
70	3		{}		-90
	};
    \addlegendentry{GIN train}
    
\addplot[color=mypurple, mark options={solid, fill=mypurple, draw=black}, mark=*, line width=0.5pt, mark size=2pt, visualization depends on=\thisrow{alignment} \as \alignment, nodes near coords, point meta=explicit symbolic,
    every node near coord/.style={anchor=\alignment, font=\scriptsize}] 
table [meta index=2]  {
x   y       label   alignment
10	91		{}		-129
20	63		{}		-90
30	35		{}		-90
40	10		{}		-90
50	8		{}		-90
60	3		{}		-90
70	3		{}		-90
	};
    \addlegendentry{GIN test}

	\end{axis}
\end{tikzpicture}
\caption{Train and test accuracy across graph sizes in the \kchoose{} problem. \newgat{} easily achieves 100\% train and test accuracy even for $k$$=$$100$ and using only a single head. GIN \cite{xu2018powerful}, although considered as more expressive than other GNNs, cannot perfectly fit the training data (with a model size of $d=128$) starting from $k$$=$$20$.}
\label{fig:kchoose-results-additional}
\end{figure*}
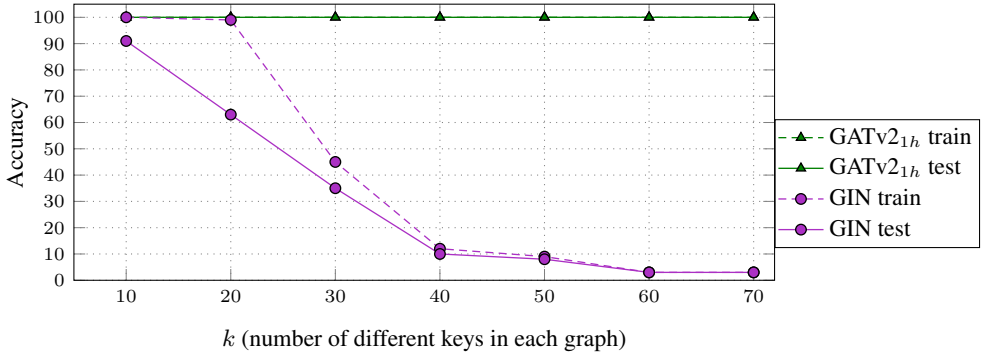

%% file: pubmed_table.tex
\subsection{Pubmed Citation Network}
\label{subsec:pubmed}
We tuned the following parameters for both GAT and \newgat{}: 
number of layers $\in \{0,1,2\}$, hidden size $\in \{8,16,32\}$, 
number of heads $\in \{1,4,8\}$, dropout $\in \{0.4, 0.6, 0.8\}$, 
bias $\in \{True, False\}$, share weights $\in \{True, False\}$, 
use residual $\in \{True, False\}$. 
\Cref{tab:pubmed1} shows the test accuracy (100 runs$\pm$stdev) using the best hyperparameters found for each model.

\begin{table*}[!h]
    \caption{Accuracy (100 runs$\pm$stdev) on Pubmed.
    \newgat{} is more accurate than GAT.}
    \centering
    \footnotesize
    \setlength\tabcolsep{5 pt}
        \begin{tabular}{@{}lc@{}}
        \toprule
         Model & Accuracy \\ 
         \midrule     
          GAT	   & 78.1\err{0.59}\\
          \newgat{}	   & \textbf{78.5}\err{0.38}\\
    \bottomrule
    \end{tabular}

    \label{tab:pubmed1}
\end{table*}

It is important to note that PubMed has only \textbf{60 training nodes}, which hinders expressive models such as \newgat{} from exploiting their approximation and generalization advantages. Still, \newgat{} is more accurate than GAT even in this small dataset. In \Cref{tab:pubmed}, we show that this difference is statistically significant (p-value $< 0.0001$).

%% file: dpgat_example_figure.tex
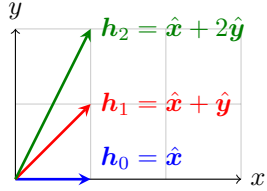
\begin{figure*}[!h]
\centering
\begin{tikzpicture}
\definecolor{ao}{rgb}{0.0, 0.5, 0.0}
\definecolor{mypurple}{HTML}{AB30C4}
  \draw[thin,gray!40] (0,0) grid (3,2);
  \draw[->] (0,0)--(3,0) node[right]{$x$};
  \draw[->] (0,0)--(0,2) node[above]{$y$};
  \draw[line width=1pt,blue,-stealth](0,0)--(1,0) node[anchor=south west]{$\vh_0=\hat{\vx}$};
  \draw[line width=1pt,red,-stealth](0,0)--(1,1) node[anchor=west]{$\vh_1=\hat{\vx}+\hat{\vy}$};
  \draw[line width=1pt,ao,-stealth](0,0)--(1,2) node[anchor=west]{$\vh_2=\hat{\vx}+2\hat{\vy}$};
\end{tikzpicture}
\caption{An example for node representations that are linearly dependent, for which \dpgat{} cannot compute dynamic attention, because no query vector $\vq\in \mathbb{R}^2$ can ``select'' $\vh_1$.}
\label{fig:dpgat-counterexample}
\end{figure*}

%% file: all_tables_with_dpgat.tex
\input{noise_fig_with_dpgat.tex}

\input{varmisuse_table_with_dpgat.tex}

\input{node-ogb-table-dp.tex}

\input{additional_qm_with_dpgat.tex}

%% file: noise_fig_with_dpgat.tex
\newcommand{\dpnoisefigsacle}{0.85} 
\newcommand{\dpnoisefigheight}{70mm} 
\begin{figure*}[h!]
    \centering
    \begin{subfigure}[b]{.54\textwidth}
		\centering
        \begin{tikzpicture}[trim axis left,trim axis right, scale=\dpnoisefigsacle]

            \definecolor{color0}{rgb}{0.917647058823529,0.917647058823529,0.949019607843137}
            \definecolor{color2}{rgb}{0.282352941176471,0.470588235294118,0.815686274509804}
            \definecolor{color1}{rgb}{0.933333333333333,0.52156862745098,0.290196078431373}
            \definecolor{color3}{rgb}{0.415686274509804,0.8,0.392156862745098}
            
            \begin{axis}[
            axis line style={black},
            legend cell align={left},
            legend style={
			},
            legend pos=north east,
            legend entries={{\dpgat},{\newgat},{GAT}},
            tick align=inside,
            tick pos=both,
            grid style={dotted, gray},
		    xlabel={$p$ -- noise ratio},
		    xlabel style={font=\large},
            xmajorgrids,
            xmin=-0.0, xmax=0.5,
            xtick style={color=white!15!black},
            ylabel={Accuracy},
            ymajorgrids,
            ymin=65.1810334331438, ymax=72.2220649577313,
            ytick style={color=white!15!black},
            ylabel near ticks,
            ]

            \path [draw=color2, fill=color2, opacity=0.2]
            (axis cs:0,71.634488639225)
            --(axis cs:0,71.339154891225)
            --(axis cs:0.05,70.242372895)
            --(axis cs:0.1,69.4833702862)
            --(axis cs:0.15,69.351447734175)
            --(axis cs:0.2,68.3770186224)
            --(axis cs:0.25,68.31221429905)
            --(axis cs:0.3,68.129738883475)
            --(axis cs:0.35,67.509041787725)
            --(axis cs:0.4,67.500562496225)
            --(axis cs:0.45,66.73229106645)
            --(axis cs:0.5,66.177005216525)
            --(axis cs:0.5,67.19779024095)
            --(axis cs:0.5,67.19779024095)
            --(axis cs:0.45,67.266662174925)
            --(axis cs:0.4,67.903011321)
            --(axis cs:0.35,68.126453399025)
            --(axis cs:0.3,68.525421105225)
            --(axis cs:0.25,68.830419160275)
            --(axis cs:0.2,69.32330129465)
            --(axis cs:0.15,69.674727781375)
            --(axis cs:0.1,70.110960331875)
            --(axis cs:0.05,70.661085662425)
            --(axis cs:0,71.634488639225)
            --cycle;
            
            \path [draw=color3, fill=color3, opacity=0.2]
            (axis cs:0,71.90201807025)
            --(axis cs:0,71.655834636175)
            --(axis cs:0.05,70.514165497)
            --(axis cs:0.1,69.80446233675)
            --(axis cs:0.15,69.1386837789)
            --(axis cs:0.2,68.738708914725)
            --(axis cs:0.25,68.2689663314)
            --(axis cs:0.3,68.266126652025)
            --(axis cs:0.35,67.809358043175)
            --(axis cs:0.4,67.28881729125)
            --(axis cs:0.45,67.025620174775)
            --(axis cs:0.5,66.3337175726)
            --(axis cs:0.5,66.73437362935)
            --(axis cs:0.5,66.73437362935)
            --(axis cs:0.45,67.506019459275)
            --(axis cs:0.4,67.687276551675)
            --(axis cs:0.35,68.1653440689)
            --(axis cs:0.3,68.72463998525)
            --(axis cs:0.25,68.988544999675)
            --(axis cs:0.2,69.20296705065)
            --(axis cs:0.15,69.482360612)
            --(axis cs:0.1,70.2378646482)
            --(axis cs:0.05,70.8456386014)
            --(axis cs:0,71.90201807025)
            --cycle;
        
            \path [draw=color1, fill=color1, opacity=0.2]
            (axis cs:0,71.690460548925)
            --(axis cs:0,71.336918506325)
            --(axis cs:0.05,70.247515105)
            --(axis cs:0.1,69.177359618825)
            --(axis cs:0.15,68.323553693375)
            --(axis cs:0.2,67.8844814293)
            --(axis cs:0.25,67.31310886385)
            --(axis cs:0.3,67.110750140475)
            --(axis cs:0.35,66.46680670075)
            --(axis cs:0.4,66.4188961789)
            --(axis cs:0.45,65.9022476761)
            --(axis cs:0.5,65.501080320625)
            --(axis cs:0.5,65.72084407885)
            --(axis cs:0.5,65.72084407885)
            --(axis cs:0.45,66.362788716425)
            --(axis cs:0.4,66.762989407975)
            --(axis cs:0.35,66.918991967375)
            --(axis cs:0.3,67.459853879)
            --(axis cs:0.25,67.701022451975)
            --(axis cs:0.2,68.230992887375)
            --(axis cs:0.15,68.694935606)
            --(axis cs:0.1,69.55477395875)
            --(axis cs:0.05,70.615602721025)
            --(axis cs:0,71.690460548925)
            --cycle;
            
            \addplot [semithick, color2, mark=square*, mark options={draw=black}]
            table {%
            0 71.484682464
            0.05 70.436392975
            0.1 69.834372712
            0.15 69.525954436
            0.2 68.918173217
            0.25 68.56305237
            0.3 68.324589539
            0.35 67.819475556
            0.4 67.69150238
            0.45 67.022201536
            0.5 66.772626497
            };
            \addplot [semithick, color3, mark=triangle*, mark options={draw=black}]
            table {%
            0 71.785281372
            0.05 70.68308487
            0.1 70.024689484
            0.15 69.324938203
            0.2 68.957263945
            0.25 68.624158478
            0.3 68.514289092
            0.35 67.972553253
            0.4 67.466616057
            0.45 67.266835785
            0.5 66.544245148
            };
            \addplot [semithick, color1, mark=*, mark options={draw=black}]
            table {%
            0 71.535913086
            0.05 70.439889525
            0.1 69.360738372
            0.15 68.494125365
            0.2 68.066168212
            0.25 67.512911225
            0.3 67.283707428
            0.35 66.718926241
            0.4 66.587866975
            0.45 66.156616211
            0.5 65.601711273
            };
            \end{axis}
        \end{tikzpicture}
        \caption{\textbf{ogbn-arxiv}}
        \label{fig:noise-arxiv}
    \end{subfigure}
	\begin{subfigure}[b]{.45\textwidth}
        \centering
\begin{tikzpicture}[trim axis left,trim axis right, scale=\dpnoisefigsacle]

    \definecolor{color0}{rgb}{0.917647058823529,0.917647058823529,0.949019607843137}
    \definecolor{color3}{rgb}{0.282352941176471,0.470588235294118,0.815686274509804}
    \definecolor{color1}{rgb}{0.933333333333333,0.52156862745098,0.290196078431373}
    \definecolor{color2}{rgb}{0.415686274509804,0.8,0.392156862745098}
    
    \begin{axis}[
    axis line style={black},
    legend cell align={left},
    legend style={
      at={(0.03,0.03)},
      anchor=south west,
    },
    legend pos=north east,
    legend entries={{\dpgat},{\newgat},{GAT}},
    tick align=inside,
    tick pos=both,
    grid style={dotted, gray},
    xlabel={$p$ -- noise ratio},
	xlabel style={font=\large},
    xmajorgrids,
    xmin=-0.0, xmax=0.5,
    xtick style={color=white!15!black},
    ylabel={Accuracy},
    ymajorgrids,
    ymin=26.973718350705, ymax=33.472489923295,
    ytick style={color=white!15!black},
    ylabel near ticks,
    ]

    \path [draw=color3, fill=color3, opacity=0.2]
    (axis cs:0,33.17709121545)
    --(axis cs:0,32.2314794264)
    --(axis cs:0.05,31.0718785079)
    --(axis cs:0.1,30.2985225215)
    --(axis cs:0.15,30.8769705163)
    --(axis cs:0.2,30.7063415766)
    --(axis cs:0.25,30.376701254425)
    --(axis cs:0.3,30.9851148075)
    --(axis cs:0.35,31.029280668175)
    --(axis cs:0.4,30.788764363325)
    --(axis cs:0.45,30.674067839175)
    --(axis cs:0.5,31.0128761057)
    --(axis cs:0.5,31.50826149505)
    --(axis cs:0.5,31.50826149505)
    --(axis cs:0.45,31.10447898135)
    --(axis cs:0.4,31.564283467375)
    --(axis cs:0.35,31.53016329805)
    --(axis cs:0.3,31.642182618725)
    --(axis cs:0.25,31.494539376725)
    --(axis cs:0.2,31.5427658097)
    --(axis cs:0.15,31.70510144145)
    --(axis cs:0.1,31.27049336785)
    --(axis cs:0.05,31.858478314875)
    --(axis cs:0,33.17709121545)
    --cycle;

    \path [draw=color2, fill=color2, opacity=0.2]
    (axis cs:0,32.863074599925)
    --(axis cs:0,32.3462407976)
    --(axis cs:0.05,30.986646844175)
    --(axis cs:0.1,30.690734253225)
    --(axis cs:0.15,30.44085464415)
    --(axis cs:0.2,30.241738361)
    --(axis cs:0.25,29.995148174775)
    --(axis cs:0.3,30.5407432914)
    --(axis cs:0.35,30.8314095492)
    --(axis cs:0.4,30.363533839575)
    --(axis cs:0.45,30.20346797815)
    --(axis cs:0.5,30.0006195131)
    --(axis cs:0.5,30.897035367925)
    --(axis cs:0.5,30.897035367925)
    --(axis cs:0.45,31.08495694525)
    --(axis cs:0.4,31.134814459325)
    --(axis cs:0.35,31.187778891325)
    --(axis cs:0.3,31.265170865775)
    --(axis cs:0.25,30.93118051835)
    --(axis cs:0.2,30.770959538125)
    --(axis cs:0.15,30.9003912823)
    --(axis cs:0.1,31.1338788615)
    --(axis cs:0.05,31.531665142625)
    --(axis cs:0,32.863074599925)
    --cycle;
    
    \path [draw=color1, fill=color1, opacity=0.2]
    (axis cs:0,32.967684883475)
    --(axis cs:0,31.295423451175)
    --(axis cs:0.05,30.106856997925)
    --(axis cs:0.1,29.2344302389)
    --(axis cs:0.15,28.6856863089)
    --(axis cs:0.2,28.39808987305)
    --(axis cs:0.25,28.195182351725)
    --(axis cs:0.3,28.22891662305)
    --(axis cs:0.35,27.829633320925)
    --(axis cs:0.4,27.289688334275)
    --(axis cs:0.45,27.6331881)
    --(axis cs:0.5,28.032070717125)
    --(axis cs:0.5,28.6780799888)
    --(axis cs:0.5,28.6780799888)
    --(axis cs:0.45,28.48289803185)
    --(axis cs:0.4,28.16217566325)
    --(axis cs:0.35,28.571519478375)
    --(axis cs:0.3,29.118154259525)
    --(axis cs:0.25,29.0393842354)
    --(axis cs:0.2,28.928038208025)
    --(axis cs:0.15,29.747818347025)
    --(axis cs:0.1,30.254184573725)
    --(axis cs:0.05,31.30647454215)
    --(axis cs:0,32.967684883475)
    --cycle;

    \addplot [semithick, color3, mark=square*, mark options={draw=black}]
    table {%
    0 32.769976043
    0.05 31.464507867
    0.1 30.751805687
    0.15 31.299029541
    0.2 31.107323074
    0.25 30.986193657
    0.3 31.334794999
    0.35 31.285199929
    0.4 31.195545579
    0.45 30.879133606
    0.5 31.277808189
    };

    \addplot [semithick, color2, mark=triangle*, mark options={draw=black}]
    table {%
    0 32.612126924
    0.05 31.252056694
    0.1 30.908223534
    0.15 30.674074174
    0.2 30.498343276
    0.25 30.436348342
    0.3 30.92515297
    0.35 31.006700324
    0.4 30.787095261
    0.45 30.648323062
    0.5 30.408688924
    };

    \addplot [semithick, color1, mark=*, mark options={draw=black}]
    table {%
    0 32.200099946
    0.05 30.725577163
    0.1 29.749159623
    0.15 29.22029648
    0.2 28.658766176
    0.25 28.620138549
    0.3 28.604163742
    0.35 28.234816741
    0.4 27.778678323
    0.45 28.085552979
    0.5 28.346407889
    };
    \end{axis}
    
    \end{tikzpicture}
     
        \caption{\textbf{ogbn-mag}}
        \label{fig:noise-mag}
    \end{subfigure}
    \caption{
    Test accuracy compared to the noise ratio:
    \newgat{} and \dpgat{} are more robust to structural noise compared to GAT. 
    Each point is an average of 10 runs, error bars show standard deviation.}
    \label{fig:noise-dpgat}
\end{figure*}

%% file: varmisuse_table_with_dpgat.tex
\begin{table*}[h!] %
    \caption{Accuracy (5 runs$\pm$stdev) on {\sc{VarMisuse}}. 
    \newgat{} is more accurate than all GNNs in both test sets, using GAT's hyperparameters. %
    ${\dagger}$ -- previously reported by \citet{brockschmidt2019graph}.
    }%
    \centering
    \footnotesize
    \setlength\tabcolsep{5 pt}
        \begin{tabular}{@{}llcc@{}}
        \toprule
         & Model & \multicolumn{1}{c}{SeenProj} & \multicolumn{1}{c}{UnseenProj} \\
         \midrule
		 \multirow{2}{*}{\shortstack[c]{No-\\Attention}}& GCN$^{\dagger}$ & 87.2\err{1.5}&  81.4\err{2.3}\\
		 & GIN$^{\dagger}$ & 87.1\err{0.1}&  81.1\err{0.9}\\
         \midrule
          \multirow{3}{*}{Attention}         
          & GAT$^{\dagger}$	   & 86.9\err{0.7}& 81.2\err{0.9}\\
          & \dpgat{} & \textbf{88.0}\err{0.8} & 81.5\err{1.2}\\
          & \newgat{}	   & \textbf{88.0}\err{1.1}& \textbf{82.8}\err{1.7}\\
\bottomrule
    \end{tabular}

    \label{tab:varmisuse-results-dpgat}
\end{table*}

%% file: node-ogb-table-dp.tex
\begin{table*}[h!]
    \centering
	\caption{Average accuracy (\Cref{tab:node-ogbn-results-dp}) and ROC-AUC (\Cref{tab:proteins-results-dp}) in node-prediction datasets (10 runs$\pm$std). In all  datasets, \newgat{} outperforms GAT.
	${\dagger}$ -- previously reported by \citet{hu2020open}.
	} 
	\begin{subtable}{0.8\textwidth}
    		\caption{}
	\centering
    \footnotesize
    \begin{tabu}{llccc}
        \toprule
        Model & Attn. Heads     & \textbf{ogbn-arxiv}                           & \textbf{ogbn-products}                & \textbf{ogbn-mag}
        \\ 
        \midrule
        GCN$^{\dagger}$ & 0 & 71.74\err{0.29} & 78.97\err{0.33} & 30.43\err{0.25} \\
        GraphSAGE$^{\dagger}$ & 0 & 71.49\err{0.27} & 78.70\err{0.36} & 31.53\err{0.15} \\
        \midrule
        \multirow{2}{*}{\shortstack[c]{GAT}} 
        & 1      & 71.59\err{0.38}                   & 79.04\err{1.54}          & 32.20\err{1.46}  \\ 
        & 8     & 71.54\err{0.30}                   & 77.23\err{2.37}          & 31.75\err{1.60}  \\ 
        \midrule
        \multirow{2}{*}{\shortstack[c]{\dpgat{}}}     & 1    & 71.52\err{0.17}                   & 76.49\err{0.78}          & \textbf{32.77}\err{0.80} \\  
		     & 8    & 71.48\err{0.26}                   & 73.53\err{0.47}          & 27.74\err{9.97} \\  
        \midrule
        \multirow{2}{*}{\shortstack[c]{\newgat{} (this work)}} 

        & 1  & 71.78\err{0.18}          & \textbf{80.63}\err{0.70} & \textbf{32.61}\err{0.44} \\ 
        & 8   & \textbf{71.87}\err{0.25}          & 78.46\err{2.45}           & 32.52\err{0.39}  \\ 
        \bottomrule
        \end{tabu}
        \label{tab:node-ogbn-results-dp}
	\end{subtable}
	\begin{subtable}{0.18\textwidth}
		\caption{}
    \centering
    \footnotesize
    \begin{tabu}{c}
        \toprule
		\textbf{ogbn-proteins}\\ 
        \midrule
        72.51\vphantom{$^{\dagger}$}\err{0.35}  \\ %
        77.68\vphantom{$^{\dagger}$}\err{0.20} \\ %
        \midrule
		70.77\err{5.79}   \\ 
        78.63\err{1.62} \\ 
        \midrule
        63.47\err{2.79} \\  
        72.88\err{0.59} \\
        \midrule

        77.23\err{3.32} \\ 
        \textbf{79.52}\err{0.55}      \\
        \bottomrule
        \end{tabu}
    
        \label{tab:proteins-results-dp}
	\end{subtable}

	\label{tab:node-ogb-dp}
\end{table*}

%% file: additional_qm_with_dpgat.tex
\renewcommand{\qmwidth}{12mm}

\begin{table*}[h!]
	\caption{Average error rates (lower is better), 5 runs $\pm$ standard deviation for each property, on the QM9 dataset. 
	The best result among GAT, \newgat{} and \dpgat{} is marked in \textbf{bold}; the globally best result among all GNNs is  marked in \textbf{\uline{bold and underline}}.	
	$\dagger$ was previously tuned and reported by \citet{brockschmidt2019graph}.}
		\centering
		\footnotesize
		\begin{tabu}{lR{\qmwidth}R{\qmwidth}R{\qmwidth}R{\qmwidth}R{\qmwidth}R{\qmwidth}R{\qmwidth}}
			\toprule
			& \multicolumn{6}{c}{Predicted Property} & \\
			Model & \multicolumn{1}{c}{1} & \multicolumn{1}{c}{2} & \multicolumn{1}{c}{3} & \multicolumn{1}{c}{4} & \multicolumn{1}{c}{5} & \multicolumn{1}{c}{6} & \multicolumn{1}{c}{7} \\
			\midrule
			\normalsize{GCN$^\dagger$} & 3.21\err{0.06} & \textbf{\uline{4.22}}\err{0.45} & 1.45\err{0.01} & 1.62\err{0.04} & 2.42\err{0.14} & 16.38\err{0.49} & 17.40\err{3.56}\\
			\normalsize{GIN$^\dagger$} & 2.64\err{0.11} & 4.67\err{0.52} & 1.42\err{0.01} & 1.50\err{0.09} & 2.27\err{0.09} & \textbf{\uline{15.63}}\err{1.40} & 12.93\err{1.81}  \\
			\midrule
			\normalsize{GAT$_{1h}$} & 3.08\err{0.08} & 7.82\err{1.42} & 1.79\err{0.10} & 3.96\err{1.51} & 3.58\err{1.03} & 35.43\err{29.9} & 116.5\err{10.65} \\		
			\normalsize{GAT$_{8h}$$^\dagger$} & 2.68\err{0.06} & 4.65\err{0.44} & 1.48\err{0.03} & 1.53\err{0.07} & 2.31\err{0.06} & 52.39\err{42.58} & 14.87\err{2.88} \\	
			\midrule
			\normalsize{\dpgat{}$_{8h}$} & \textbf{\uline{2.63}}\err{0.09} & 4.37\err{0.13} & 1.44\err{0.07} & \textbf{\uline{1.40}}\err{0.03} & \textbf{\uline{2.10}}\err{0.07} & 32.59\err{34.77} & \textbf{\uline{11.66}}\err{1.00} \\		
			\normalsize{\dpgat{}$_{1h}$} & 3.20\err{0.17} & 8.35\err{0.78} & 1.71\err{0.03} & 2.17\err{0.14} & 2.88\err{0.12} & 25.21\err{2.86} & 65.79\err{39.84}  \\		
			\midrule
			\normalsize{\newgat{}$_{1h}$} & 3.04\err{0.06} & 6.38\err{0.62} & 1.68\err{0.04} & 2.18\err{0.61} & 2.82\err{0.25} & 20.56\err{0.70} & 77.13\err{37.93}  \\
			\normalsize{\newgat{}$_{8h}$} & 2.65\err{0.05} & \textbf{4.28}\err{0.27} & \textbf{\uline{1.41}}\err{0.04} & 1.47\err{0.03} & 2.29\err{0.15} & 16.37\err{0.97} & 14.03\err{1.39}  \\		
			\bottomrule
			\rule{0pt}{1ex}    
		\end{tabu}
			\begin{tabu}{lR{\qmwidth}R{\qmwidth}R{\qmwidth}R{\qmwidth}R{\qmwidth}R{\qmwidth}R{\qmwidth}}
			\toprule
			& \multicolumn{6}{c}{Predicted Property} & Rel. to \\
			Model & \multicolumn{1}{c}{8} & \multicolumn{1}{c}{9} & \multicolumn{1}{c}{10} & \multicolumn{1}{c}{11} & \multicolumn{1}{c}{12} & \multicolumn{1}{c}{13} &  GAT$_{8h}$ \\
			\midrule
			\normalsize{GCN$^\dagger$} & 7.82\err{0.80} & 8.24\err{1.25} & 9.05\err{1.21} & 7.00\err{1.51} & 3.93\err{0.48} & \textbf{\uline{1.02}}\err{0.05} & -1.5\% \\
			\normalsize{GIN$^\dagger$} & \textbf{\uline{5.88}}\err{1.01} & 18.71\err{23.36} & \textbf{\uline{5.62}}\err{0.81} & \textbf{\uline{5.38}}\err{0.75} & \textbf{\uline{3.53}}\err{0.37} & 1.05\err{0.11} & -2.3\% \\
			\midrule
			\normalsize{GAT$_{1h}$} & 28.10\err{16.45} & 20.80\err{13.40} & 15.80\err{5.87} & 10.80\err{2.18} & 5.37\err{0.26} & 3.11\err{0.14} & +134.1\% \\		
			\normalsize{GAT$_{8h}$$^\dagger$} & 7.61\err{0.46} & 6.86\err{0.53} & 7.64\err{0.92} & 6.54\err{0.36} & 4.11\err{0.27} & 1.48\err{0.87} & +0\% \\
			\midrule
			\normalsize{\dpgat{}$_{1h}$} & 12.93\err{1.70} & 13.32\err{2.39} & 14.42\err{1.95} & 13.83\err{2.55} & 6.37\err{0.28} & 3.28\err{1.16} & +77.9\%\\			
			\normalsize{\dpgat{}$_{8h}$} & 6.95\err{0.32} & 7.09\err{0.59} & 7.30\err{0.66} & 6.52\err{0.61} & 3.76\err{0.21} & \textbf{1.18}\err{0.33} & -9.7\%\\
			\midrule
			\normalsize{\newgat{}$_{1h}$} & 10.19\err{0.63} & 22.56\err{17.46} & 15.04\err{4.58} & 22.94\err{17.34} & 5.23\err{0.36} & 2.46\err{0.65} & +91.6\%\\			
			\normalsize{\newgat{}$_{8h}$} & \textbf{6.07}\err{0.77}& \textbf{\uline{6.28}}\err{0.83} & \textbf{6.60}\err{0.79} & \textbf{5.97}\err{0.94} & \textbf{3.57}\err{0.36} & 1.59\err{0.96} & \textbf{\uline{-11.5}}\% \\		
			\bottomrule
		\end{tabu}

	\label{tab:qm-results-appendix-with-dp}
	\end{table*}

%% file: all_tables_with_pvalue.tex
\section{Statistical Significance}
\label{subsec:significance}
Here we report the statistical significance of the strongest \newgat{} and GAT models of the
experiments reported in \Cref{sec:eval}.

\input{noise_fig_with_pvalue.tex}

\input{varmisuse_table_with_pvalue.tex}

\input{node-ogb-table-pvalue.tex}

\input{link-tables-pvalue.tex}

\input{additional_qm_with_pvalue.tex}

%% file: noise_fig_with_pvalue.tex
\begin{figure*}[h!]
    \centering
    \begin{subfigure}[b]{.54\textwidth}
		\centering
        \begin{tikzpicture}[trim axis left,trim axis right, scale=\noisefigsacle]

            \definecolor{color0}{rgb}{0.917647058823529,0.917647058823529,0.949019607843137}
            \definecolor{color1}{rgb}{0.282352941176471,0.470588235294118,0.815686274509804}
            \definecolor{color2}{rgb}{0.933333333333333,0.52156862745098,0.290196078431373}
            \definecolor{color3}{rgb}{0.415686274509804,0.8,0.392156862745098}
            
            \begin{axis}[
            axis line style={black},
            legend cell align={left},
            legend style={
			},
            legend pos=north east,
            legend entries={{\newgat{} (p-value)},{GAT}},
            tick align=inside,
            tick pos=both,
            grid style={dotted, gray},
		    xlabel={noise ratio},
            xmajorgrids,
            xmin=-0.0, xmax=0.5,
            xtick style={color=white!15!black},
            ylabel={Accuracy},
            ymajorgrids,
            ymin=65.1810334331438, ymax=72.2220649577313,
            ytick style={color=white!15!black},
            ylabel near ticks,
			height=\noisefigheight,
            ]

            \path [draw=color3, fill=color3, opacity=0.2]
            (axis cs:0,71.90201807025)
            --(axis cs:0,71.655834636175)
            --(axis cs:0.05,70.514165497)
            --(axis cs:0.1,69.80446233675)
            --(axis cs:0.15,69.1386837789)
            --(axis cs:0.2,68.738708914725)
            --(axis cs:0.25,68.2689663314)
            --(axis cs:0.3,68.266126652025)
            --(axis cs:0.35,67.809358043175)
            --(axis cs:0.4,67.28881729125)
            --(axis cs:0.45,67.025620174775)
            --(axis cs:0.5,66.3337175726)
            --(axis cs:0.5,66.73437362935)
            --(axis cs:0.5,66.73437362935)
            --(axis cs:0.45,67.506019459275)
            --(axis cs:0.4,67.687276551675)
            --(axis cs:0.35,68.1653440689)
            --(axis cs:0.3,68.72463998525)
            --(axis cs:0.25,68.988544999675)
            --(axis cs:0.2,69.20296705065)
            --(axis cs:0.15,69.482360612)
            --(axis cs:0.1,70.2378646482)
            --(axis cs:0.05,70.8456386014)
            --(axis cs:0,71.90201807025)
            --cycle;
        
            \path [draw=color2, fill=color2, opacity=0.2]
            (axis cs:0,71.690460548925)
            --(axis cs:0,71.336918506325)
            --(axis cs:0.05,70.247515105)
            --(axis cs:0.1,69.177359618825)
            --(axis cs:0.15,68.323553693375)
            --(axis cs:0.2,67.8844814293)
            --(axis cs:0.25,67.31310886385)
            --(axis cs:0.3,67.110750140475)
            --(axis cs:0.35,66.46680670075)
            --(axis cs:0.4,66.4188961789)
            --(axis cs:0.45,65.9022476761)
            --(axis cs:0.5,65.501080320625)
            --(axis cs:0.5,65.72084407885)
            --(axis cs:0.5,65.72084407885)
            --(axis cs:0.45,66.362788716425)
            --(axis cs:0.4,66.762989407975)
            --(axis cs:0.35,66.918991967375)
            --(axis cs:0.3,67.459853879)
            --(axis cs:0.25,67.701022451975)
            --(axis cs:0.2,68.230992887375)
            --(axis cs:0.15,68.694935606)
            --(axis cs:0.1,69.55477395875)
            --(axis cs:0.05,70.615602721025)
            --(axis cs:0,71.690460548925)
            --cycle;

            \addplot[semithick, color3, mark options={draw=black}, mark=triangle*, 
            visualization depends on=\thisrow{alignment} \as \alignment, 
            nodes near coords, point meta=explicit symbolic,
            every node near coord/.style={anchor=\alignment, font=\scriptsize, color=black, rotate=-45, yshift=0, xshift=0}] 
            table [meta index=2]  {
            x   y   label   alignment
            0  71.785281372    {}  90
            0.05   70.68308487   {} 90
            0.1  70.024689484  0.0002 -10
            0.15    69.324938203  <0.0001 180
            0.2 68.957263945 <0.0001 -10
            0.25 68.624158478 0.0001 180
            0.3 68.514289092 <0.0001 0
            0.35 67.972553253 <0.0001 180
            0.4 67.466616057 p-value<0.0001 -10
            0.45 67.266835785 <0.0001 -10
            0.5 66.544245148 {} 90
            };

            \addplot [semithick, color2, mark=*, mark options={draw=black}]
            table {%
            0 71.535913086
            0.05 70.439889525
            0.1 69.360738372
            0.15 68.494125365
            0.2 68.066168212
            0.25 67.512911225
            0.3 67.283707428
            0.35 66.718926241
            0.4 66.587866975
            0.45 66.156616211
            0.5 65.601711273
            };
            \end{axis}
        \end{tikzpicture}
        \caption{\textbf{ogbn-arxiv}}
        \label{fig:noise-arxiv-pvalue}
    \end{subfigure}
	\begin{subfigure}[b]{.45\textwidth}
        \centering
\begin{tikzpicture}[trim axis left,trim axis right, scale=\noisefigsacle]

    \definecolor{color0}{rgb}{0.917647058823529,0.917647058823529,0.949019607843137}
    \definecolor{color1}{rgb}{0.282352941176471,0.470588235294118,0.815686274509804}
    \definecolor{color3}{rgb}{0.933333333333333,0.52156862745098,0.290196078431373}
    \definecolor{color2}{rgb}{0.415686274509804,0.8,0.392156862745098}
    
    \begin{axis}[
    axis line style={black},
    legend cell align={left},
    legend style={
      at={(0.03,0.03)},
      anchor=south west,
    },
    legend pos=north east,
    legend entries={{\newgat{} (p-value)},{GAT}},
    tick align=inside,
    tick pos=both,
    grid style={dotted, gray},
    xlabel={noise ratio},
    xmajorgrids,
    xmin=-0.0, xmax=0.5,
    xtick style={color=white!15!black},
    ylabel={Accuracy},
    ymajorgrids,
    ymin=26.973718350705, ymax=33.472489923295,
    ytick style={color=white!15!black},
    ylabel near ticks,
	height=\noisefigheight,
    ]

    \path [draw=color2, fill=color2, opacity=0.2]
    (axis cs:0,32.863074599925)
    --(axis cs:0,32.3462407976)
    --(axis cs:0.05,30.986646844175)
    --(axis cs:0.1,30.690734253225)
    --(axis cs:0.15,30.44085464415)
    --(axis cs:0.2,30.241738361)
    --(axis cs:0.25,29.995148174775)
    --(axis cs:0.3,30.5407432914)
    --(axis cs:0.35,30.8314095492)
    --(axis cs:0.4,30.363533839575)
    --(axis cs:0.45,30.20346797815)
    --(axis cs:0.5,30.0006195131)
    --(axis cs:0.5,30.897035367925)
    --(axis cs:0.5,30.897035367925)
    --(axis cs:0.45,31.08495694525)
    --(axis cs:0.4,31.134814459325)
    --(axis cs:0.35,31.187778891325)
    --(axis cs:0.3,31.265170865775)
    --(axis cs:0.25,30.93118051835)
    --(axis cs:0.2,30.770959538125)
    --(axis cs:0.15,30.9003912823)
    --(axis cs:0.1,31.1338788615)
    --(axis cs:0.05,31.531665142625)
    --(axis cs:0,32.863074599925)
    --cycle;
    
    \path [draw=color3, fill=color3, opacity=0.2]
    (axis cs:0,32.967684883475)
    --(axis cs:0,31.295423451175)
    --(axis cs:0.05,30.106856997925)
    --(axis cs:0.1,29.2344302389)
    --(axis cs:0.15,28.6856863089)
    --(axis cs:0.2,28.39808987305)
    --(axis cs:0.25,28.195182351725)
    --(axis cs:0.3,28.22891662305)
    --(axis cs:0.35,27.829633320925)
    --(axis cs:0.4,27.289688334275)
    --(axis cs:0.45,27.6331881)
    --(axis cs:0.5,28.032070717125)
    --(axis cs:0.5,28.6780799888)
    --(axis cs:0.5,28.6780799888)
    --(axis cs:0.45,28.48289803185)
    --(axis cs:0.4,28.16217566325)
    --(axis cs:0.35,28.571519478375)
    --(axis cs:0.3,29.118154259525)
    --(axis cs:0.25,29.0393842354)
    --(axis cs:0.2,28.928038208025)
    --(axis cs:0.15,29.747818347025)
    --(axis cs:0.1,30.254184573725)
    --(axis cs:0.05,31.30647454215)
    --(axis cs:0,32.967684883475)
    --cycle;

    \addplot[semithick, color2, mark options={draw=black}, mark=triangle*, 
    visualization depends on=\thisrow{alignment} \as \alignment, 
    nodes near coords, point meta=explicit symbolic,
    every node near coord/.style={anchor=\alignment, font=\scriptsize, color=black, rotate=-45, yshift=0, xshift=0}] 
    table [meta index=2]  {
    x   y   label   alignment
    0   32.612126924    {}  90
    0.05    31.252056694   {} 90
    0.1 30.908223534  0.0006 -10
    0.15    30.674074174  0.0001 180
    0.2 30.498343276 p-value<0.0001 0
    0.25 30.436348342 <0.0001 180
    0.3 30.92515297 <0.0001 180
    0.35 31.006700324 <0.0001 180
    0.4 30.787095261 <0.0001 0
    0.45 30.648323062 <0.0001 160
    0.5 30.408688924 {} 90
    };

    \addplot [semithick, color3, mark=*, mark options={draw=black}]
    table {%
    0 32.200099946
    0.05 30.725577163
    0.1 29.749159623
    0.15 29.22029648
    0.2 28.658766176
    0.25 28.620138549
    0.3 28.604163742
    0.35 28.234816741
    0.4 27.778678323
    0.45 28.085552979
    0.5 28.346407889
    };
    \end{axis}
    
    \end{tikzpicture}
     
        \caption{\textbf{ogbn-mag}}
        \label{fig:noise-mag-pvalue}
    \end{subfigure}
    \caption{
    Test accuracy  and statistical significance compared to the noise ratio:
    \newgat{} is more robust to structural noise compared to GAT. 
    Each point is an average of 10 runs, error bars show standard deviation.}
    \label{fig:noise}
\end{figure*}
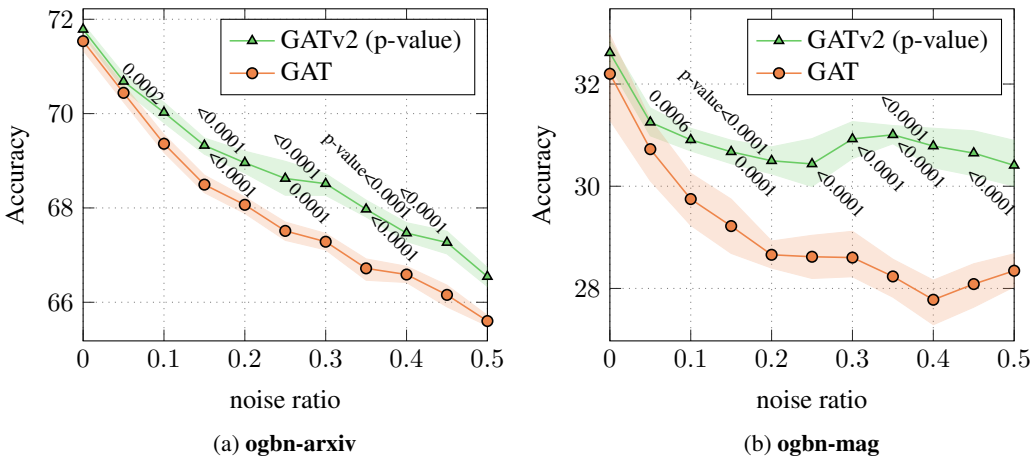

%% file: varmisuse_table_with_pvalue.tex
\begin{table*}[h!] 
    \caption{Accuracy (5 runs$\pm$stdev) on {\sc{VarMisuse}}. 
    \newgat{} is more accurate than all GNNs in both test sets, using GAT's hyperparameters. %
    ${\dagger}$ -- previously reported by \citet{brockschmidt2019graph}. 
    }
    \centering
    \footnotesize
    \setlength\tabcolsep{5 pt}
        \begin{tabular}{@{}lcc@{}}
        \toprule
         Model & \multicolumn{1}{c}{SeenProj} & \multicolumn{1}{c}{UnseenProj} \\
         \midrule    
          GAT$^{\dagger}$	   & 86.9\err{0.7}& 81.2\err{0.9}\\
          \newgat{}	   & \textbf{88.0}\err{1.1}& \textbf{82.8}\err{1.7}\\
          \midrule
          p-value & 0.048 & 0.049\\
         
\bottomrule
    \end{tabular}

    \label{tab:varmisuse-results-pvalue}
\end{table*}

%% file: node-ogb-table-pvalue.tex
\begin{table*}[!h]
        \caption{Accuracy (100 runs$\pm$stdev) on Pubmed.
        \newgat{} is more accurate than GAT.}
        \centering
        \footnotesize
        \setlength\tabcolsep{5 pt}
            \begin{tabular}{@{}lc@{}}
            \toprule
             Model & Accuracy \\ 
             \midrule     
              GAT	   & 78.1\err{0.59}\\
              \newgat{}	   & \textbf{78.5}\err{0.38}\\
              \midrule
              p-value & < 0.0001\\
        \bottomrule
        \end{tabular}
    
        \label{tab:pubmed}
    \end{table*}

\begin{table*}[h!]
    \centering
	\caption{Average accuracy (\Cref{tab:node-ogbn-results-pvalue}) and ROC-AUC (\Cref{tab:proteins-results-pvalue}) in node-prediction datasets (30 runs$\pm$std).
        We report on the best GAT / \newgat{} from \Cref{tab:node-ogb}.
	} 
	\begin{subtable}{0.8\textwidth}
    		\caption{}
	\centering
    \footnotesize
    \begin{tabu}{lccc}
        \toprule
        Model    & \textbf{ogbn-arxiv}                        & \textbf{ogbn-products}                & \textbf{ogbn-mag} 
        \\ 
        \midrule
        GAT     & 71.65\err{0.38}                   & 79.04\err{1.54}          & 32.36\err{1.10}  \\ 

        \newgat{}  & \textbf{71.93}\err{0.35}          & \textbf{80.63}\err{0.70} & \textbf{33.01}\err{0.41} \\ 
        \midrule
        p-value  & 0.0022 & <0.0001 &  0.0018\\
        \bottomrule
        \end{tabu}
        \label{tab:node-ogbn-results-pvalue}
	\end{subtable}
	\begin{subtable}{0.18\textwidth}
		\caption{}
    \centering
    \footnotesize
    \begin{tabu}{c}
        \toprule
\textbf{ogbn-proteins} \\
        \midrule
        78.29 \err{1.59} \\ 

        \textbf{78.96}\err{1.19}      \\
        \midrule
        0.0349  \\
        \bottomrule
        \end{tabu}
    
        \label{tab:proteins-results-pvalue}
	\end{subtable}

	\label{tab:node-ogb-pvalue}
\end{table*}

%% file: link-tables-pvalue.tex
\begin{table*}[h!]
\caption{Average Hits@50 %
(\Cref{tab:collab-results-pvalue}) and mean reciprocal rank (MRR) (\Cref{tab:citation2-results-pvalue}) in link-prediction benchmarks from OGB (30 runs$\pm$std). 
We report on the best GAT / \newgat{} from \Cref{tab:link-results}.
}
\centering
	\begin{subtable}{0.6\textwidth}
    \caption{}
    \centering
    \footnotesize
    \begin{tabu}{lcc}
        \toprule
        \rowfont{\footnotesize}
            & \multicolumn{2}{c}{\textbf{ogbl-collab}}\\
       Model & \multicolumn{1}{c}{w/o val edges}& \multicolumn{1}{c}{w/ val edges} \\
        \midrule
        GAT & 42.24\err{2.26} & 46.02\err{4.09} \\
        \newgat{} & \textbf{43.82}\err{2.24} & \textbf{49.06}\err{2.50} \\
        \midrule
        p-value & 0.0043 & 0.0005\\
        \bottomrule
    \end{tabu}
    \label{tab:collab-results-pvalue}
	\end{subtable}
	\begin{subtable}{0.15\textwidth}
        \caption{}
    \centering
    \footnotesize
    \begin{tabular}{c}
        \toprule
        \textbf{ogbl-citation2}\\ 
        \\
        \midrule
        79.91\err{0.13} \\ 
        \textbf{80.20}\err{0.62} \\ 
        \midrule
        0.0075 \\
        \bottomrule
        \end{tabular}
         \label{tab:citation2-results-pvalue}
	\end{subtable}

\label{tab:link-results-dpgat-pvalue}
\end{table*}

%% file: additional_qm_with_pvalue.tex
	\begin{table*}[h!]
		\caption{Average error rates (lower is better), 20 runs $\pm$ standard deviation for each property, on the QM9 dataset. 
		We report on GAT and \newgat{} with 8 attention heads.
		}
		\centering
		\footnotesize
		\begin{tabu}{lC{\qmwidth}C{\qmwidth}C{\qmwidth}C{\qmwidth}C{\qmwidth}C{\qmwidth}C{\qmwidth}}
			\toprule
			& \multicolumn{6}{c}{Predicted Property} & \\
			Model & \multicolumn{1}{c}{1} & \multicolumn{1}{c}{2} & \multicolumn{1}{c}{3} & \multicolumn{1}{c}{4} & \multicolumn{1}{c}{5} & \multicolumn{1}{c}{6} & \multicolumn{1}{c}{7} \\
			\midrule
			GAT & 2.74\err{0.08} & 4.73\err{0.40} & 1.47\err{0.06} & 1.53\err{0.06} & 2.44\err{0.60} & 55.21\err{42.33} & 25.36\err{31.42}  \\
			\newgat{} & \textbf{2.67}\err{0.08} & \textbf{4.28}\err{0.23} & \textbf{1.43}\err{0.05} & \textbf{1.51}\err{0.07} & \textbf{2.21}\err{0.08} & \textbf{16.64}\err{1.17} & \textbf{13.61}\err{1.68} \\
			\midrule
			p-value & 0.0043 & <0.0001 & 0.0138  & 0.1691 & 0.0487 & 0.0001 & 0.0516  \\
			\bottomrule
			\rule{0pt}{1ex}    
		\end{tabu}
			\begin{tabu}{lC{\qmwidth}C{\qmwidth}C{\qmwidth}C{\qmwidth}C{\qmwidth}C{\qmwidth}}
			\toprule
			& \multicolumn{6}{c}{Predicted Property} \\
			Model & \multicolumn{1}{c}{8} & \multicolumn{1}{c}{9} & \multicolumn{1}{c}{10} & \multicolumn{1}{c}{11} & \multicolumn{1}{c}{12} & \multicolumn{1}{c}{13} \\
			\midrule
		GAT & 7.36\err{0.87} & 6.79\err{0.86} & 7.36\err{0.93} & 6.69\err{0.86} & 4.10\err{0.29} & 1.51\err{0.84} \\

			\newgat{} & \textbf{6.13}\err{0.59} & \textbf{6.33}\err{0.82} & \textbf{6.37}\err{0.86} & \textbf{5.95}\err{0.62} & \textbf{3.66}\err{0.29} & \textbf{1.09}\err{0.85} \\
			\midrule
			p-value & <0.0001 & 0.0458  & 0.0006 & 0.0017 & <0.0001 & 0.0621  \\
			\bottomrule
		\end{tabu}
	\label{tab:qm-results-appendix-with-pvalue}
	\end{table*}

%% file: complexity.tex
\clearpage
\section{Complexity Analysis}
\label{sec:complexity}

We repeat the definitions of GAT, \newgat{} and \dpgat{}:
\begin{align}
&			 \mathrm{GAT} \text{ \cite{velic2018graph}:} &
e\left(\vh_i, \vh_j\right)= &
	\mathrm{LeakyReLU}
	\left(
		\va^{\top}\cdot\left[\mW\vh_i \| \mW\vh_j\right]
	\right)
	\label{eq:gat-vs-gat2}  
\\
	&			 \mathrm{\newgat{}} \text{ (our fixed version):} &
 e
 \left(\vh_{i}, \vh_{j}\right)  = &
\va
^{\top}
	\mathrm{LeakyReLU}
	\left(
		\mW \cdot \left[\vh_{i} \| \vh_{j}\right] 
	\right) \\
	&			 \mathrm{\dpgat{}} \text{ \cite{vaswani2017attention}:} &
	e\left(\vh_i, \vh_j\right) = &
		\left(\left(\vh_i^{\top}\mQ\right) \cdot \left(\vh_j^{\top}\mK \right)^{\top} \right)
		/ \sqrt{d'} 
	\label{eq:complexity}
\end{align}

\subsection{Time Complexity}
\label{subsec:time-complexity}

\para{GAT} As noted by \citet{velic2018graph}, the time complexity of a single GAT head may be expressed as $\mathcal{O}\left(\abs{\mathcal{V}}dd' + \abs{\mathcal{E}}d'\right)$. 
Because of GAT's static attention, this computation can be further optimized, by merging the linear layer $\va_1$ with $\mW$, merging $\va_2$ with $\mW$, and only then compute $\va_{\{1,2\}}^{\top}\mW\vh_i$ for every $i\in \mathcal{V}$.

\para{\newgat{}} require the same computational cost as GAT's declared complexity: $\mathcal{O}\left(\abs{\mathcal{V}}dd' + \abs{\mathcal{E}}d'\right)$: 
we denote $\mW=\left[\mW_1 \| \mW_2\right]$, where $\mW_1\in \mathbb{R}^{d' \times d}$ and $\mW_2^{d' \times d}$ contain the left half and right half of the columns of $\mW$, respectively.
We can first compute $\mW_1 \vh_i$ and $\mW_2 \vh_j$ for every $i,j\in \mathcal{V}$. This takes $\mathcal{O}\left(\abs{\mathcal{V}}dd'\right)$. 

Then, for every edge $\left(j,i\right)$, we compute $\mathrm{LeakyReLU}
	\left(
		\mW \cdot \left[\vh_{i} \| \vh_{j}\right] 
	\right)$ using the precomputed $\mW_1 \vh_i$ and $\mW_2 \vh_j$, since 
	$\mW \cdot \left[\vh_{i} \| \vh_{j}\right] = \mW_1 \vh_{i} + \mW_2 \vh_{j} $. This takes $\mathcal{O}\left(\abs{\mathcal{E}}d'\right)$.
	
	Finally, computing the results of the linear layer $\va$ takes additional $\mathcal{O}\left(\abs{\mathcal{E}}d'\right)$ time, and overall $\mathcal{O}\left(\abs{\mathcal{V}}dd' + \abs{\mathcal{E}}d'\right)$.

\para{\dpgat{}} also takes the same time.
We can first compute $\vh_i^{\top}\mQ$ and $\vh_j^{\top}\mK$ for every $i,j\in \mathcal{V}$. This takes $\mathcal{O}\left(\abs{\mathcal{V}}dd'\right)$. 
Computing the dot-product $\left(\vh_i^{\top}\mQ\right)  \left(\vh_j^{\top}\mK \right)^{\top}$ for every edge $\left(j,i\right)$ takes additional $\mathcal{O}\left(\abs{\mathcal{E}}d'\right)$ time, and overall $\mathcal{O}\left(\abs{\mathcal{V}}dd' + \abs{\mathcal{E}}d'\right)$.

\subsection{Parametric Complexity}
\label{subsec:parametric-cost}

\begin{table*}[!h]
    \centering
        \begin{tabular}{lccc}
        \toprule
         & GAT & \newgat{} & \dpgat{} \\
         \midrule
	Official & $2d'+ dd'$ & $d'+ 2dd'$ & $2dd_{k} + dd'$ \\
	In our experiments & $2d'+ dd'$ & $d'+ dd'$ & $2dd'$  \\
	\bottomrule
    \end{tabular}
    \caption{Number of parameters for each GNN type, in a single layer and a single attention head.
    }%
    \label{tab:params}
\end{table*}

All parametric costs are summarized in \Cref{tab:params}.
All following calculations refer to a single layer having a single attention head, omitting bias vectors.

\para{GAT} has learned vector and a matrix: $\va\in \mathbb{R}^{2d'}$ and $\mW\in \mathbb{R}^{d'\times d}$, thus overall $2d'+ dd'$ learned parameters.

\para{\newgat{}} has a matrix that is twice larger: $\mW\in \mathbb{R}^{d'\times 2d}$, because it is applied on the concatenation $\left[\vh_{i} \| \vh_{j}\right]$. 
Thus, the overall number of learned parameters is  $d'+ 2dd'$.
However in our experiments, to rule out the increased number of parameters over GAT as the source of empirical difference, we constrained $\mW=\left[\mW' \| \mW'\right]$, and thus the number of parameters were $d'+ dd'$.

\para{\dpgat{}} has $\mQ$ and $\mK$ matrices of sizes $dd_{k}$ each, and additional $dd'$ parameters in the value matrix $\mV$, thus $2dd_{k} + dd'$ parameters overall.
However in our experiments, we constrained $\mQ=\mK$ and set $d_{k}=d'$, and thus the number of parameters is only $2dd'$.